\documentclass[lettersize,journal]{IEEEtran}

\usepackage{algorithmic}
\usepackage{array}
\usepackage{textcomp}
\usepackage{url}
\usepackage{verbatim}
\usepackage{cite}
\usepackage[utf8]{inputenc} 
\usepackage[T1]{fontenc}    
\usepackage{url}            
\usepackage{booktabs}       
\usepackage{amsfonts}       
\usepackage{amsmath} 
\usepackage{amssymb}  
\usepackage{amsthm}
\usepackage{mathrsfs}
\usepackage[ruled]{algorithm2e}
\usepackage{bm}
\usepackage{stfloats}
\usepackage{subcaption}
\usepackage{graphicx}
\usepackage{nicefrac}       
\usepackage{microtype}      

\newcommand{\mA}{\mathcal{A}}

\newcommand{\mC}{\mathcal{C}}

\newcommand{\mS}{\mathcal{S}}

\newcommand{\mL}{\mathcal{L}}

\newcommand{\R}{\mathbb{R}}
\newcommand{\mbE}{\mathbb{E}}

\newcommand{\mO}{\mathcal{O}}

\newcommand{\tx}{\tilde{x}}

\newcommand{\dd}{\mathrm{d}}
\newcommand{\dX}{\mathrm{d}X}

\newcommand{\dt}{\mathrm{d}t}
\newcommand{\dtau}{\mathrm{d}\tau}
\newcommand{\dW}{\mathrm{d}W}

\newcommand{\tX}{\tilde{X}}

\newcommand{\Until}{\mathsf{U}}
\newcommand{\Always}{\square}
\newcommand{\Eventually}{\lozenge}
\newcommand{\real}{\mathbb{R}}

\newcommand{\amgf}[1]{\Phi(#1)}

\newcommand{\defeq}{:=}
\newcommand{\tr}[1]{\text{tr}(#1)}

\newcommand{\innerp}[1]{\left< #1 \right>}
\newcommand{\expect}[1]{\mathbb{E}\left( #1 \right)}
\newcommand{\expectw}[2]{\mathbb{E}_{#1}\left( #2 \right)}
\newcommand{\prob}[1]{\mathbb{P}\left( #1 \right)}

\renewcommand{\top}{\mathsf{T}}

\newtheorem{definition}{Definition}
\newtheorem{lemma}{Lemma}
\newtheorem{thm}{Theorem}
\newtheorem{proposition}{Proposition}

\newtheorem{assumption}{Assumption}
\newtheorem{remark}{Remark}
\newtheorem{problem}{Problem}

\begin{document}

\title{Feedback Motion Planning for Stochastic Nonlinear Systems with Signal Temporal Logic Specifications}

\author{Liqian Ma$^{*}$, Zishun Liu$^{*}$, Glen Chou, and Yongxin Chen
\thanks{The first two authors share equal contributions. This work is partially supported by NSF 2450378, NSF 2206576, and AFOSR FA9550-25-1-0169.}%
\thanks{The authors are with Georgia Institute of Technology, Atlanta, GA 30332 
        {\tt\small \{mlq\}\{zliu910\}\{chou\}\{yongchen\}@gatech.edu}}%
}


\maketitle

\begin{abstract}
We study feedback motion planning for continuous-time stochastic nonlinear systems under signal temporal logic (STL) specifications. We propose a framework that synthesizes control policies for chance-constrained STL trajectory optimization problems, with the goal of ensuring that the closed-loop stochastic system satisfies a given STL formula with high probability (e.g., 99.99\%). Our approach is based on a predicate erosion strategy that transforms the intractable stochastic problem into a deterministic STL trajectory optimization problem with tightened STL formula constraints. The amount of erosion is determined by a probabilistic reachable tube (PRT) that bounds the deviation between the stochastic trajectory and an associated nominal trajectory. To compute such bounds, we leverage contraction theory and feedback design, and develop several tracking controllers. This yields a complete feedback motion planning pipeline which can be implemented by numerical optimizations. We demonstrate the efficacy and versatility of the proposed framework through simulations on several robotic systems and through experiments on a real-world quadrupedal robot, and show that it is less conservative and achieves higher specification satisfaction probability than representative baselines.
\end{abstract}

\begin{IEEEkeywords}
Motion and Path Planning, Formal Methods in Robotics and Automation, Probability and Statistical Methods, Signal Temporal Logic (STL)
\end{IEEEkeywords}

\section{Introduction}

\begin{figure}[t!]
    \centering
    \includegraphics[width=\linewidth]{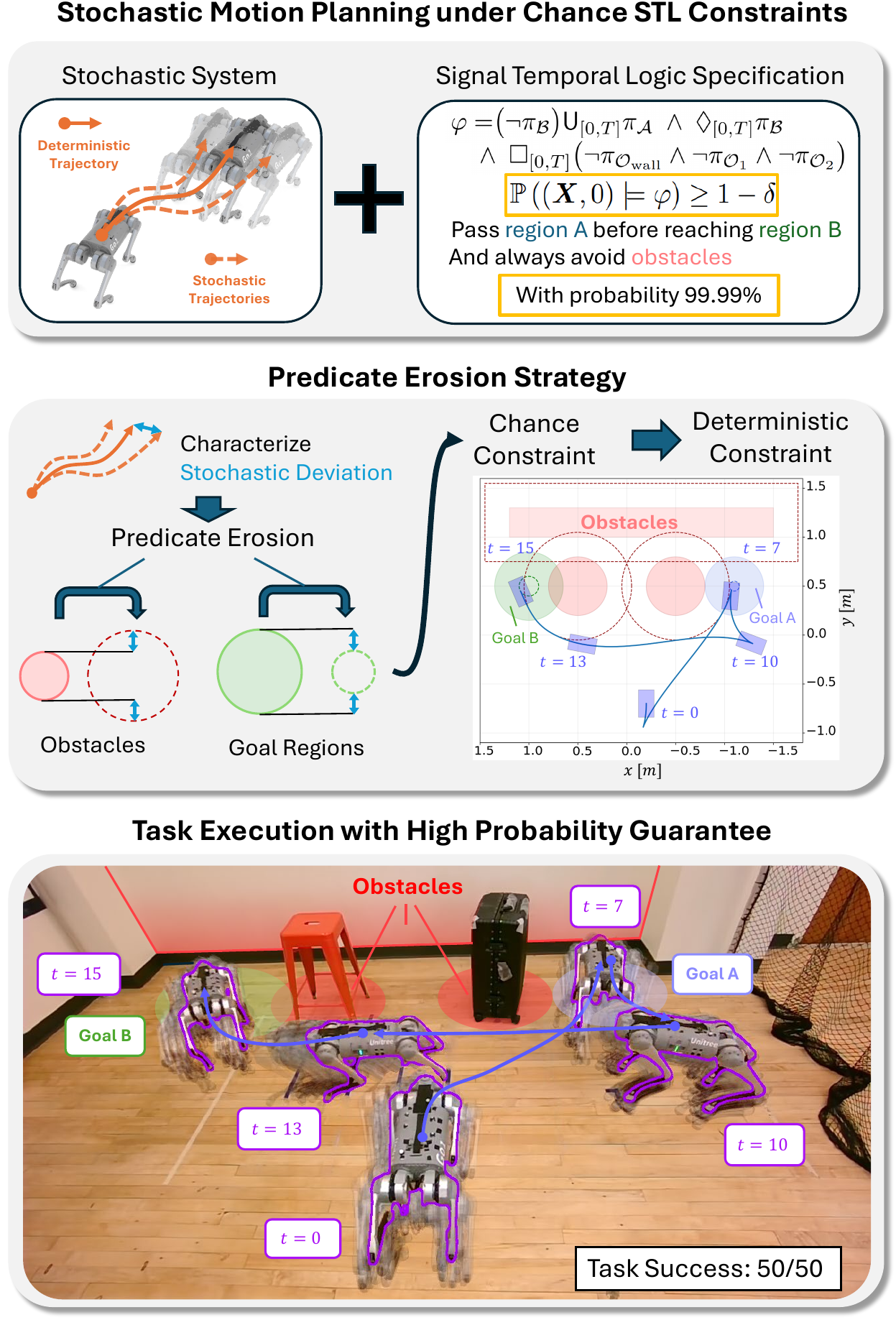}
    \caption{\looseness-1\textbf{Overview.}
    \textbf{Top:} Problem setup. We consider a stochastic nonlinear system subject to a chance-constrained STL specification.
    \textbf{Middle:} Predicate erosion strategy. We characterize the stochastic deviation between the closed-loop stochastic trajectory and its associated deterministic trajectory by a PRT, and use this tube to erode the obstacle and goal predicates. This converts the original chance-constrained problem into a deterministic STL-constrained motion planning problem with a tightened STL formula. 
    \textbf{Bottom:} Task execution. The quadrupedal robot successfully completes the task in all 50 trials. One representative trajectory is highlighted in purple, while the remaining trajectories are overlaid transparently for visualization.
    }
    \label{fig:teaser}
\end{figure}

\IEEEPARstart{R}{obotic} systems operate under uncertainty and they must satisfy precisely stated behavioral requirements under safety- and mission-critical constraints.
To reason about such requirements rigorously, practitioners employ \emph{formal specifications} that capture desired properties as statements over system trajectories. 
While classical safe-set specifications only encode spatial constraints (the state must never enter an unsafe region)~\cite{prajna2004safety}, many real-world tasks are inherently \emph{spatio–temporal}: they involve deadlines, sequencing, or persistence requirements (e.g., “visit region A before B and remain there for at least $\tau$ seconds”). 
\emph{Signal Temporal Logic} (STL) provides an intuitive way to specify a wide range of these high-level properties in continuous time and over real-valued signals~\cite{maler2004monitoring}. Therefore, it is well-suited to applications in continuous-time dynamical systems, such as robotics~\cite{pant2018fly, meng2023signal, chen2018signal, silano2021power, sewlia2022cooperative}.

In this context, verification asks whether a given closed-loop system satisfies an STL formula~\cite{roehm2016stl, lercher2024using, ma2025verification}, whereas \emph{control synthesis} aims to \emph{design} control strategies that guarantee satisfaction of the STL specification. The latter is crucial for deploying dependable autonomy at scale and is the focus of this work. Multiple approaches to control synthesis under STL specifications have been developed in recent years. 
One prominent line of research encodes satisfaction at sampled time instants by introducing binary variables, and then solves a mixed-integer convex program~\cite{raman2014model, sadraddini2015robust}. However, this type of encoding cannot handle nonlinear dynamics or nonlinear predicates. 
Another line of work formulates control synthesis as a continuous optimization by encoding the STL specification via its robustness metric, replacing the nonsmooth $\min/\max$ operators with differentiable surrogates, and solving the resulting problem using standard nonlinear programming solvers~\cite{gilpin2020smooth, leung2023backpropagation, meng2023signal, han2025exactsmoothreformulationstrajectory}, but the resulting nonconvex optimization problems can suffer from convergence to local minima.
Beyond sampled-time formulations, direct continuous-time synthesis methods have also been studied. Early approaches address restricted fragments of STL and typically exclude nested temporal structure (specifications where temporal operators are composed within one another)~\cite{lindemann2018control,lindemann2019control}, while more recent work extends control barrier function (CBF) based synthesis to nested STL formulas ~\cite{yu2024continuous} by encoding satisfaction via backward reachability, a step that can become computationally demanding and hence challenging to scale with system dimension. Moreover, \cite{kurtz2023temporal} leverages graphs of convex sets for continuous-time temporal-logic path planning, but does not enforce that the planned trajectories are feasible for the underlying nonlinear dynamics.

In real-world dynamical systems, unmodeled dynamics and external perturbations are unavoidable, so STL control synthesis methods commonly treat uncertainty as unknown-but-bounded disturbances and reason about the worst case. Representative approaches include robust model predictive control (MPC) that propagates disturbance in robustness scores through linear system evolution~\cite{sadraddini2015robust}, interval semantics using Hamilton-Jacobi reachability~\cite{verhagen2024robust}, and learning-based methods using adversarial training~\cite{yaghoubi2019worst}.

In many practical settings, disturbances are better modeled as unbounded stochastic noise (e.g., sensor or actuator noise) rather than unknown-but-bounded inputs. Under such uncertainty, methods that focus on worst-case bounds become overly conservative, as they must account for extreme, low-probability events. We therefore adopt a probabilistic formulation, seeking to design controls that ensure the STL specification holds with high probability (e.g., $\ge 99.99\%$). While probabilistic approaches exist, they exhibit significant limitations when scaled to continuous-time or nonlinear systems. For instance, \cite{farahani2018shrinking, yang2023distributed} translate chance constraints onto individual predicates, but are restricted to discrete-time linear systems. Other works \cite{vlahakis2024probabilistic, ma2025verification} use probabilistic reachable sets to tighten predicates; however, this reduction relies on the union bound (Boole’s inequality), which breaks down in continuous time where STL satisfaction requires an intersection over uncountably many time instants. Recently, \cite{yao2025model} addressed continuous-time STL by utilizing the unscented transform to approximate uncertainty propagation. Because this is an approximation, it lacks formal probabilistic guarantees and is fundamentally limited to a restricted fragment of STL with affine predicates. Consequently, efficiently solving motion planning problems for continuous-time stochastic systems under general STL specifications remains a significant open challenge.

In this paper, we address this gap by proposing a systematic feedback motion planning framework that maps the original stochastic problem to a deterministic STL trajectory optimization (STLTO) problem with tightened STL formulas. The key idea is to construct an nominal trajectory, design a feedback tracking controller around it, and use a \emph{probabilistic reachable tube} (PRT) to bound the stochastic deviation from the nominal trajectory. This tube is then used to erode each predicate in the STL formula so that satisfaction of the tightened deterministic specification implies high-probability satisfaction of the original chance-constrained specification. Our contributions are summarized as follows.
\begin{itemize}
\item We introduce a \textbf{continuous-time STL erosion strategy} that systematically transforms chance-constrained STL trajectory optimization problems into deterministic ones.
\item We develop a \textbf{complete pipeline for feedback control synthesis under probabilistic STL specifications}, enabling practitioners to design feedback controllers for stochastic systems with rigorous probabilistic guarantees.
\item We demonstrate the scalability and applicability of our framework through extended experiments on various robotic systems in both simulation and real world, and show that it is less conservative and achieves higher specification satisfaction probability than baselines.
\end{itemize}

The rest of the paper is organized as follows. Section~\ref{sec: problem formulation} reviews the STL semantics used throughout the paper and formulates continuous-time stochastic STL trajectory optimization. Section~\ref{sec: STLTO} introduces the predicate-erosion and derives a deterministic surrogate of the chance-constrained problem. Section~\ref{sec: prob_bound} develops the contraction-based probabilistic reachable tube bound that underlies the erosion procedure. Section~\ref{sec: feedback control} presents the tracking-controller designs and the feedback motion planning pipeline. Section~\ref{sec:numerical_optimization} describes how to implement the numerical optimizations in practice. Section~\ref{sec:experiments} and Section~\ref{sec:experiments-quadruped} validates the proposed approach on simulated systems and quadrupedal robot experiments.

\textit{Notations.} We use $\real$ and $\real_{\ge 0}$ to denote the sets of real numbers and nonnegative real numbers, respectively. The notation $\real^n$ denotes the $n$-dimensional real vector space, and $\real^{m\times n}$ denotes the set of real $m\times n$ matrices. Bold symbols such as $\boldsymbol{x}$ and $\boldsymbol{X}$ denote trajectories or signals over time, while $x_t$ and $X_t$ denote their values at time $t$. Specifically, capitalized variables (such as $\boldsymbol{X}$ and $X_t$) represent stochastic trajectories and states, whereas lowercase variables (such as $\boldsymbol{x}$ and $x_t$) represent their deterministic counterparts. We use $I_n$ to denote the $n\times n$ identity matrix, and for symmetric matrices $A$ and $B$, $A\preceq B$ means that $B-A$ is positive semidefinite. $\mathbb{S}_{++}^n$ denotes the set of $n\times n$ symmetric positive definite matrices. For a vector $x\in\real^n$, $\|x\|$ denotes the Euclidean norm, and $\|x\|_{M}\defeq \sqrt{x^\top M x}$ denotes the weighted norm induced by a positive definite matrix $M$. Given sets $A$ and $B$, their Minkowski sum and Pontryagin difference are defined by $A\oplus B\defeq \{a+b:\ a\in A,\ b\in B\}$ and $A\ominus B\defeq \{x:\ x+B\subseteq A\}$, respectively. Finally, we use $\prob{\cdot}$ and $\mbE\{\cdot\}$ to denote probability and expectation.
\section{Problem Formulation}
\label{sec: problem formulation}

In this section, we first present the dynamic model and the syntax and semantics of STL. Then, we formulate the stochastic STL trajectory optimization problem.

\subsection{Dynamic Model}
We consider the continuous-time stochastic system
\begin{equation}
    \mathrm{d}X_t = f_t(X_t, u_t)\,\mathrm{d}t + g_t(X_t)\,\mathrm{d}W_t
    \label{eq: stochastic_system}
\end{equation}
where $X_t\in \real^n$ is the state at time $t$, $u_t \in \mathcal{U}$ is the control input applied at time $t$, $\mathcal{U} \subset \real^p$ is a bounded set of allowed control inputs, $f_t(X_t, u_t)\,\mathrm{d}t$ is the drift term with $f: \real^n\times\real^p\times\real_{\geq0}\to\real^n$, $g_t(X_t)\,\mathrm{d}W_t$ is the diffusion term, and $W_t \in \real^m$ is an $m$-dimensional Wiener process (Brownian motion). We impose standard Lipschitz continuity and linear growth conditions \cite[Theorem 5.2.1]{BO:13} to guarantee the existence of a solution to \eqref{eq: stochastic_system}. For the diffusion term, we assume that $g_t(X_t)$ is bounded for all $t$.
\begin{assumption}\label{as: boundness}
    For the stochastic system~\eqref{eq: stochastic_system}, there exists $\sigma>0$ such that $g_t(X)g_t(X)^{\top}\preceq \sigma^2 I_n$ for any $t\geq0$ and $X\in\R^n$.
\end{assumption}

We also introduce its deterministic counterpart. Consider the continuous-time deterministic system
\begin{equation}\label{eq: deterministic_system}
    \dot{x}_{t} = f_t(x_t, u_t),
\end{equation}
where \(f_t\) is identical to the function defined in \eqref{eq: stochastic_system}, then \eqref{eq: deterministic_system} can be regarded as the noise-free version of \eqref{eq: stochastic_system}. Let $\boldsymbol{u}: \real_{\geq 0} \rightarrow \real^p$ be an input signal to the system~\eqref{eq: deterministic_system}, the resulting trajectory of~\eqref{eq: deterministic_system} is a curve $\boldsymbol{x}: \real_{\geq 0}\rightarrow\real^n$. We use $\boldsymbol{x}(t) = x_t$ to denote the state reached at time $t$ on the trajectory.

\subsection{STL Syntax and Semantics}
We use STL~\cite{maler2004monitoring} to specify the desired spatio-temporal properties of the system trajectory $\boldsymbol{x}$. 
STL uses logical operators (negation $\neg$, conjunction $\wedge$, disjunction $\vee$) and temporal operators (\textit{always} $\Always_{[t_1,t_2]}$, \textit{eventually} $\Eventually_{[t_1,t_2]}$, \textit{until} $\Until_{[t_1,t_2]}$) to recursively define specifications. The syntax of an STL formula is given as
\begin{multline}
\label{eq: standard STL syntax}
    \varphi := \pi {}\mid{} \neg\varphi {}\mid{} \varphi_1\wedge\varphi_2 {}\mid{} \varphi_1\vee\varphi_2 {}\mid{} \\
  \Always_{[t_1,t_2]}\varphi {}\mid{} \Eventually_{[t_1,t_2]}\varphi\mid{} \varphi_1 \Until_{[t_1,t_2]}\varphi_2 {},
\end{multline}
where $\pi:=(\mu_\pi(x)\geq 0)$ is a predicate defined by a function $\mu_\pi: \real^n \rightarrow \real$, $\varphi, \varphi_1, \varphi_2$ are STL formulas, and $0\leq t_1\leq t_2<\infty$ define time intervals which the temporal operators consider. An STL formula $\varphi$ is recursively constructed using this syntax.

The Boolean semantics of an STL formula $\varphi$ are defined over a system trajectory $\boldsymbol{x}$~\cite{maler2004monitoring}, which we summarize in Definition~\ref{def: boolean-stl}.
\begin{definition}[Boolean semantics]\label{def: boolean-stl}
We write $(\boldsymbol{x},t)\models\varphi$ to denote that $\boldsymbol{x}$ satisfies $\varphi$ at time $t$.
Let $I=[t_1,t_2]\subseteq\real_{\ge0}$ and $t+I\defeq\{\,t+\tau:\tau\in I\,\}$.
The semantics are given recursively by
\begin{align*}
(\boldsymbol{x},t)\models \pi 
&\ \Leftrightarrow\ \mu_\pi(\boldsymbol{x}(t))\ge 0,\\
(\boldsymbol{x},t)\models \neg\varphi 
&\ \Leftrightarrow\ \neg\big((\boldsymbol{x},t)\models \varphi\big),\\
(\boldsymbol{x},t)\models \varphi_1\wedge\varphi_2
&\ \Leftrightarrow\ (\boldsymbol{x},t)\models \varphi_1\ \wedge\ (\boldsymbol{x},t)\models \varphi_2,\\
(\boldsymbol{x},t)\models \varphi_1\vee\varphi_2
&\ \Leftrightarrow\ (\boldsymbol{x},t)\models \varphi_1\ \vee\ (\boldsymbol{x},t)\models \varphi_2,\\
(\boldsymbol{x},t)\models \Always_{I}\varphi
&\ \Leftrightarrow\ \forall \tau\in t+I,\ (\boldsymbol{x},\tau)\models \varphi,\\
(\boldsymbol{x},t)\models \Eventually_{I}\varphi
&\ \Leftrightarrow\ \exists \tau\in t+I,\ (\boldsymbol{x},\tau)\models \varphi,\\
(\boldsymbol{x},t)\models \varphi_1\Until_{I}\varphi_2
&\ \Leftrightarrow\ \exists \tau\in t+I,\ (\boldsymbol{x},\tau)\models \varphi_2\ \wedge \\
&\qquad \qquad \qquad\forall \tau'\in[t,\tau],(\boldsymbol{x},\tau')\models \varphi_1.
\end{align*}
\end{definition}

STL admits a \emph{robustness (quantitative)} semantics that assigns a real score to each formula and trajectory segment~\cite{fainekos2009robustness}. 
Intuitively, a positive value of robustness indicates satisfaction, and a negative value indicates a violation, with the magnitude reflecting the margin of satisfaction/violation. The quantitative semantics are introduced in Definition~\ref{def: robustness-stl}.

\begin{definition}[Quantitative semantics]\label{def: robustness-stl}
The robustness of $\varphi$ with respect to $\boldsymbol{x}$ at time $t$ is a real value $\rho_\varphi(\boldsymbol{x},t)\in\real$ defined recursively:
\begin{align*}
\rho_{\pi}(\boldsymbol{x},t) 
&\defeq \mu_\pi(\boldsymbol{x}(t)),\\
\rho_{\neg\varphi}(\boldsymbol{x},t) 
&\defeq -\,\rho_{\varphi}(\boldsymbol{x},t),\\
\rho_{\varphi_1\wedge\varphi_2}(\boldsymbol{x},t) 
&\defeq \min\{\rho_{\varphi_1}(\boldsymbol{x},t),\ \rho_{\varphi_2}(\boldsymbol{x},t)\},\\
\rho_{\varphi_1\vee\varphi_2}(\boldsymbol{x},t) 
&\defeq \max\{\rho_{\varphi_1}(\boldsymbol{x},t),\ \rho_{\varphi_2}(\boldsymbol{x},t)\},\\
\rho_{\Always_{I}\varphi}(\boldsymbol{x},t) 
&\defeq \min_{\tau\in t+I}\ \rho_{\varphi}(\boldsymbol{x},\tau),\\
\rho_{\Eventually_{I}\varphi}(\boldsymbol{x},t) 
&\defeq \max_{\tau\in t+I}\ \rho_{\varphi}(\boldsymbol{x},\tau),\\
\rho_{\varphi_1\Until_{I}\varphi_2}(\boldsymbol{x},t) 
&\defeq \max_{\tau\in t+I}\ \min\!\Big\{\rho_{\varphi_2}(\boldsymbol{x},\tau),\ \min_{\tau'\in[t,\tau]}\rho_{\varphi_1}(\boldsymbol{x},\tau')\Big\}.
\end{align*}
Finally, Boolean satisfaction is recovered from the sign of robustness:
\[
(\boldsymbol{x},t)\models \varphi \ \Leftrightarrow\ \rho_\varphi(\boldsymbol{x},t)\ge 0.
\]
\end{definition}

For the ease of analysis, we assume that all STL formulas are converted into a negation-free form, i.e., without explicit negation operators. Such a conversion is always possible by first converting the formula into Negation Normal Form~\cite{fainekos2009robustness} and, if necessary, introducing new predicates with reversed inequalities~\cite[Proposition~2]{belta2019formal, sadraddini2015robust}. We also assume that the horizon of the formula is bounded. Specifically, the horizon of an STL formula refers to the minimum time duration over which a trajectory must be defined to fully evaluate the formula's satisfaction. This formula horizon must be shorter than the overall horizon of the motion planning problem.

\subsection{Stochastic STL Motion Planning Problem}
Given a dynamic system and a desired STL specification, the goal of motion planning is to derive a control input signal $\boldsymbol{u}$ which can result in a trajectory $\boldsymbol{x}$ satisfying the specification. This can be effectively formulated as a trajectory optimization problem, where the satisfaction of the STL specification is added as a constraint.

First consider the deterministic STLTO problem. Given a trajectory $\boldsymbol{x}$ of system~\eqref{eq: deterministic_system} with initial state $x_0$, we say $\boldsymbol{x}$ satisfies a STL specification $\varphi$ if $(\boldsymbol{x}, 0)\models \varphi$. The goal of deterministic STLTO is to derive a control input curve $\boldsymbol{u}$ which can minimize some cost while ensuring the constraint $(\boldsymbol{x}, 0)\models \varphi$ is satisfied~\cite{raman2014model, farahani2015robust}.

However, such a definition of specification satisfaction can be overly restrictive for a stochastic trajectory $\boldsymbol{X}$ of system~\eqref{eq: stochastic_system}. Since the stochastic noise can go unbounded, the stochastic state $X_t$ can be unbounded as well, which leads to inevitable violation of the STL specification. Therefore, we consider the probabilistic satisfaction of an STL specification in the stochastic trajectory optimization problem. Given a user-defined probability tolerance $\delta\in[0,1]$, a STL specification $\varphi$, we say the trajectory $\boldsymbol{X}$ controlled by $\boldsymbol{u}$ satisfies $\varphi$ with probability $1-\delta$ if
\begin{equation}\label{eq: stoch_stl}
    \prob{(\boldsymbol{X}, 0)\models \varphi} \geq 1-\delta,
\end{equation}
where $\delta$ is usually a small value. For example, if we choose $\delta=10^{-4}$, then the constraint requires that the resulting stochastic trajectory $\boldsymbol{X}$ satisfies $\varphi$ with probability at least $99.99\%$.

Given the stage cost function $\mathcal{L}_t(X_t, u_t)$, the terminal cost $\Phi_T(X_T)$, and the initial state $X_0 = x_0$, the stochastic trajectory optimization problem under STL specifications can be formulated as the following stochastic optimization problem with the chance constraint~\eqref{eq: stoch_stl}
\begin{subequations}\label{eq: stoch_STLTO}
    \begin{align}
        \min_{\bm{X,u}}\quad &\mbE\left\{\int_{0}^{T} \mL_t(X_{t},u_{t})\mathrm{d}t + \Phi_T(X_T)\right\}\\
        \mbox{s.t.}\quad &\mathrm{d}X_t = f_t(X_t, u_t)\mathrm{d}t + g_t(X_t)\mathrm{d}W_t, \\
        &\boldsymbol{X}(0)=x_0, \\
        &\prob{(\boldsymbol{X}, 0)\models \varphi} \geq 1-\delta, \label{eq: stoch_STLTO_STL_con}\\
        & u_t\in\mathcal{U}.
    \end{align}
\end{subequations}

To the best of our knowledge, no existing method systematically solves trajectory optimization for \emph{continuous-time} stochastic systems under \emph{general} STL specifications with explicit high-probability guarantees. The difficulty comes from the need to enforce a trajectory-level chance constraint over uncountably many time instants while simultaneously handling the logical and temporal structure of STL specifications. This motivates the following problem.

\begin{problem}[Chance-constrained STL trajectory optimization] \label{problem 1}
Given the stochastic STLTO in~\eqref{eq: stoch_STLTO} with risk level $\delta\in(0,1)$ and STL formula $\varphi$ over horizon $[0,T]$, develop a scalable algorithm that computes a control $\boldsymbol{u}$ while ensuring the chance constraint~\eqref{eq: stoch_STLTO_STL_con} holds.
\end{problem}

At a high level, our solution strategy is as follows. We first replace the original stochastic STL constraint by a deterministic one through \emph{predicate erosion}: each predicate in the STL formula is tightened by a margin large enough to absorb stochastic tracking error. To compute this margin, we associate each stochastic trajectories with their nominal deterministic trajectory and bound their deviation by a PRT. We then design a feedback tracking controller so that the closed-loop system is contracting, which allows us to derive a tight PRT. The next three sections develop these ingredients in order: predicate erosion and deterministic reduction in Section~\ref{sec: STLTO}, contraction-based PRT construction in Section~\ref{sec: prob_bound}, and feedback controller design in Section~\ref{sec: feedback control}.

\section{Stochastic Trajectory Optimization with STL Constraints}
\label{sec: STLTO}

The stochastic optimization problem~\eqref{eq: stoch_STLTO} is fundamentally challenging due to the presence of chance constraints. Even in the simpler case of safe-set constraints, where~\eqref{eq: stoch_STLTO_STL_con} is replaced by $\prob{X_t\in \mathcal{X}_{\text{safe}}, \forall t} \geq 1-\delta$ and there are no temporal requirements, ensuring continuous-time satisfaction remains challenging. Existing approaches for stochastic trajectory optimization under safe-set constraints typically seek to derive tractable \emph{deterministic} surrogates for the expectations and probability constraints~\cite{nakka2019trajectory,nakka2022trajectory,9993003, liu2025trajectory}. When STL specifications are introduced, this complexity is exacerbated due to the coupling of spatial and temporal requirements.

To effectively solve Problem~\ref{problem 1}, we introduce \emph{predicate erosion}, a systematic strategy that transforms the chance-constrained formulation~\eqref{eq: stoch_STLTO} into a tractable, deterministic trajectory optimization problem.

\subsection{Predicate Erosion Strategy}

We interpret the deterministic system~\eqref{eq: deterministic_system} as the nominal, noise-free counterpart to the stochastic dynamics~\eqref{eq: stochastic_system}. Given the same initial state and control curve $\boldsymbol{u}$, we refer to the resulting state trajectories of \eqref{eq: deterministic_system} and \eqref{eq: stochastic_system} as the \emph{nominal trajectory} $\boldsymbol{x}$ and the \emph{stochastic trajectory} $\boldsymbol{X}$, respectively. We say $\boldsymbol{x}$ and $\boldsymbol{X}$ are associated trajectories. The intuition underlying our approach is that a stochastic trajectory $\boldsymbol{X}$ will fluctuate around its deterministic counterpart $\boldsymbol{x}$ with high probability.

Building upon this intuition, we extend the \textit{predicate erosion} strategy~\cite{ma2025verification} to the continuous-time domain to efficiently handle the stochastic STLTO problem~\eqref{eq: stoch_STLTO}. Consider a predicate function $\mu_\pi(\cdot)$ defining a valid region via its superlevel set $\mC = \{x \in \real^n \mid \mu_\pi(x) \geq 0\}$. We construct a smaller set $\tilde{\mC} \subset \mC$ by systematically eroding the boundary of $\mC$. The core insight is that if the nominal state $x_t$ lies in the eroded set $\tilde{\mC}$, the corresponding stochastic state $X_t$ is guaranteed to reside within the original set $\mC$ with high probability.

This erosion mechanism establishes a tractable, deterministic proxy for the original stochastic STL constraint. Rather than directly evaluating intractable chance constraints over the stochastic trajectory, we constrain the nominal trajectory to satisfy a tightened STL specification, wherein every constituent predicate is replaced by its eroded counterpart. Consequently, if the deterministic trajectory satisfies this eroded specification, the stochastic trajectory will satisfy the original STL formula with the desired high probability. As intuitively illustrated in Fig.~\ref{fig:teaser}, this process translates to shrinking goal regions and expanding obstacle boundaries. By planning against this deterministic proxy, we systematically solve the stochastic STLTO problem while retaining rigorous probabilistic guarantees.

To formalize our method, we introduce the following definitions. Definition~\ref{def: stochastic deviation} states the stochastic deviation of the stochastic trajectory with respect to the deterministic trajectory.

\begin{definition}\label{def: stochastic deviation}
Let $\boldsymbol{X}$ and $\boldsymbol{x}$ denote the associated stochastic and deterministic trajectories in $\real^n$, respectively. 
The stochastic fluctuation at time $t$ is
\[
e_t \defeq X_t - x_t,
\]
and the stochastic deviation is the weighted norm of the fluctuation with a weight (metric) matrix $M_t\in\mathbb{S}_{++}^n$:
\[
\|e_t\| \defeq \|X_t - x_t\|_{M_t}.
\]
The fluctuation trajectory is $\boldsymbol{e} = \boldsymbol{X} - \boldsymbol{x}$.
\end{definition}

The reason why we consider a weighted norm $\|\cdot\|_{M_t}$ is that the standard Euclidean norm $\|\cdot\|$ may become conservative when the system is high-dimensional \cite{vershynin2018high}. The choice of $M_t$ will be discussed in Section \ref{sec: prob_bound}. We use \textit{probabilistic reachable tube (PRT)} to characterize the trajectory-level evolution of the stochastic fluctuation.

\begin{definition}\label{def: PRT}
Given associated trajectories $\boldsymbol{X}$ and $\boldsymbol{x}$, a finite horizon $[0,T]$, a probability level $\delta\in(0,1)$, a positive definite matrix $M_t\in\mathbb{S}_{++}^n$ and a radius curve $r_{\delta,t}:[0,T]\to\R_{\ge 0}$, define for each $t\in[0,T]$ the closed ellipsoid
\[
E_t \defeq \{y\in\real^n:\ \|y\|_{M_t}\le r_{\delta,t}\}.
\]
The indexed family $\boldsymbol{E} \defeq \{E_t\}_{t\in[0,T]}$ is a PRT for $e_t = X_t - x_t$ with probability $1-\delta$ if
\begin{equation}\label{eq: def PRT}
    \begin{split}
        &\prob{\boldsymbol{e}\in\boldsymbol{E}}\defeq~\prob{e_t\in E_t, \ \forall t\in[0,T]}\\
        =~&\prob{\|X_t-x_t\|_{M_t}\le r_{\delta,t}, \ \forall t\in[0,T]}\geq 1-\delta,
    \end{split}
\end{equation}
That is, with probability at least $1-\delta$, the entire fluctuation trajectory remains inside the tube $\boldsymbol{E}$ over $[0,T]$.
\end{definition}

\begin{remark}
    The construction of a tight PRT will be discussed in Section~\ref{sec: prob_bound}. In general, due to underactuation and limited control authority, the system trajectories may not be sufficiently contractive to suppress stochastic deviations, so disturbances can accumulate over time. Consequently, $r_{\delta,t}$ is typically a non-decreasing function of $t$, and $\boldsymbol{E}$ is an expanding tube.

\end{remark}

To derive the predicate-erosion strategy, we extend the predicates in the STL syntax~\eqref{eq: standard STL syntax} to \emph{time-varying} predicates, since $E_t$ (and thus $\mathcal{C}\ominus E_t$) depends explicitly on $t$.

\begin{definition}\label{def: tv predicate}
Let $\mu_\pi:\real^n\times[0,T]\to\real$ be a time-varying predicate function. 
The time-varying atomic predicate at time $t$ is
\[
\pi_t:\ (\mu_\pi(x,t)\ge 0).
\]
The time-varying STL syntax is obtained from the standard STL syntax~\eqref{eq: standard STL syntax} by replacing atomic predicates with $\pi_t$. 
Its Boolean semantics are defined accordingly by
\[
(\boldsymbol{x}, t)\models \pi_t \ \Leftrightarrow\ \mu_\pi(\boldsymbol{x}(t),t)\ge 0.
\]
\end{definition}

Using the PRT as Definition~\ref{def: PRT}, we can erode a predicate in the STL formula $\varphi$ to get a tightened and time-varying predicate.

\begin{proposition}[Predicate erosion]\label{prop: predicate erosion}
Let $\pi$ be a (time-invariant) predicate with superlevel set $\mathcal{C} \defeq \{x\in\real^n:\ \mu_\pi(x)\ge 0\}$. 
Given a PRT $\boldsymbol{E}=\{E_t\}_{t\in[0,T]}$, define the time-varying tightened predicate $\tilde\pi_t$ by
\[
(\boldsymbol{x}, t)\models \tilde\pi_t \ \Longleftrightarrow\ x_t\in \mathcal{C}\ominus E_t.
\]
Then, for every $t\in[0,T]$ and every $\boldsymbol{e}\in\boldsymbol{E}$,
\[
(\boldsymbol{x}, t)\models \tilde \pi_t \ \Longrightarrow\ (\boldsymbol{x}+\boldsymbol{e}, t)\models \pi.
\]
\end{proposition}

\begin{proof}
Fix $t\in[0,T]$ and $\boldsymbol{e}\in\boldsymbol{E}$, so $e_t\in E_{t}$. 
If $(\boldsymbol{x}, t)\models \tilde\pi_t$, then by definition $x_t\in \mathcal{C}\ominus E_t$, i.e., $x_t\oplus E_t\subseteq \mathcal{C}$. 
Since $e_t\in E_t$, it follows that $x_t+e_t\in \mathcal{C}$, equivalently $(\boldsymbol{x}+\boldsymbol{e}, t)\models \pi$.
\end{proof}

\begin{remark}
In practice, one can use the largest $E_t$ to erode the predicate which results in a time-invariant $\tilde{\pi}$. This avoids the need to introduce time-varying predicates but may lead to slightly more conservative results (the eroded superlevel set becomes smaller). We use time-varying predicates in the theoretical results for completeness.
\end{remark}

Proposition~\ref{prop: predicate erosion} extends \cite[Proposition 1]{ma2025verification} to the continuous-time setting and to the time-varying predicates in Definition~\ref{def: tv predicate}. 
The next theorem shows that a continuous-time probabilistic STL constraint can be reduced to a deterministic one by eroding every predicate in the STL formula using Proposition~\ref{prop: predicate erosion}.

\begin{thm}[STL formula erosion]\label{thm: stl erosion}
Let $\boldsymbol{X}$ and $\boldsymbol{x}$ be associated stochastic and deterministic trajectories, and let $\boldsymbol{E}$ be a $1-\delta$ PRT as in Definition~\ref{def: PRT}. 
Consider an STL formula $\varphi$ constructed from predicates $\pi^1,\dots,\pi^m$ with superlevel sets $\mathcal{C}^1,\dots,\mathcal{C}^m$.
Construct the eroded formula $\tilde{\varphi}$ by replacing each atomic predicate $\pi_i$ with the time-varying predicate $\tilde{\pi}_{i,t}$ whose superlevel set is $\mathcal{C}^i\ominus E_t$, leaving all logical and temporal operators unchanged.
Then
\[
(\boldsymbol{x}, 0)\models \tilde{\varphi}\ \Longrightarrow\ \prob{(\boldsymbol{X}, 0)\models \varphi}\ \ge\ 1-\delta.
\]
\end{thm}

\begin{proof}
By the PRT definition~\eqref{eq: def PRT}, it suffices to prove the following deterministically: for every fluctuation trajectory $\boldsymbol{e}\in\boldsymbol{E}$,
\begin{equation*}\label{eq: erosion-goal}
(\boldsymbol{x}, 0)\models \tilde{\varphi}\ \Longrightarrow\ (\boldsymbol{x}+\boldsymbol{e}, 0)\models \varphi.
\end{equation*}

For convenience, we prove the following \emph{stronger} claim:
\begin{equation}\label{eq: erosion-strong}
(\boldsymbol{x}, t)\models \tilde{\varphi}\ \Longrightarrow(\boldsymbol{x}+\boldsymbol{e}, t)\models \varphi
,\ \forall t\in[0,T],\ \forall \boldsymbol{e}\in\boldsymbol{E}.
\end{equation}

Since the STL formulas are recursively defined, we prove \eqref{eq: erosion-strong} by structural induction on $\varphi$.
(We assume $\varphi$ is negation-free as stated in Section~\ref{sec: problem formulation}.)

\emph{Base case (atomic predicates).}
For any atomic $\pi_i$ and its erosion $\tilde{\pi}_{i,t}$, Proposition~\ref{prop: predicate erosion} yields
\[
(\boldsymbol{x}, t)\models \tilde{\pi}_{i,t} \ \Rightarrow\ (\boldsymbol{x}+\boldsymbol{e}, t)\models \pi_i
\quad \text{for all } t \text{ and all } \boldsymbol{e}\in\boldsymbol{E}.
\]

\emph{Induction step.} Assume \eqref{eq: erosion-strong} holds for $\varphi$ and $\psi$ and their eroded versions $\tilde{\varphi}$ and $\tilde{\psi}$. We verify it for compositions:

\begin{enumerate}
    \item {
        \noindent{Logical operators.}
        \begin{itemize}[leftmargin=0.5em]
        \item Conjunction: If $(\boldsymbol{x},t)\models \widetilde{(\varphi\wedge \psi)}$ then $(\boldsymbol{x},t)\models \tilde{\varphi}$ and $(\boldsymbol{x},t)\models \tilde{\psi}$. By the induction hypothesis (IH), $(\boldsymbol{x}+\boldsymbol{e},t)\models \varphi$ and $(\boldsymbol{x}+\boldsymbol{e},t)\models \psi$, hence $(\boldsymbol{x}+\boldsymbol{e},t)\models \varphi\wedge\psi$.
        \item Disjunction: If $(\boldsymbol{x},t)\models \widetilde{(\varphi\vee \psi)}$ then $(\boldsymbol{x},t)\models \tilde{\varphi}$ or $(\boldsymbol{x},t)\models \tilde{\psi}$. The IH implies the corresponding satisfaction for $(\boldsymbol{x}+\boldsymbol{e},t)$, hence $(\boldsymbol{x}+\boldsymbol{e},t)\models \varphi\vee\psi$.
        \end{itemize}
    }
    \item{
        \noindent{Temporal operators.}
        Let $I=[t_1,t_2]$ with $0\le t_1\le t_2\le T$.

        \begin{itemize}[leftmargin=0.5em]
        \item Always: If $(\boldsymbol{x},t)\models \widetilde{\Always_I \varphi}$, then for all $\tau\in t+I$, $(\boldsymbol{x},\tau)\models \tilde{\varphi}$. By IH, for all $\tau\in t+I$ and any $\boldsymbol{e}\in\boldsymbol{E}$, $(\boldsymbol{x}+\boldsymbol{e},\tau)\models \varphi$. Hence $(\boldsymbol{x}+\boldsymbol{e},t)\models \Always_I \varphi$.

        \item Eventually: If $(\boldsymbol{x},t)\models \widetilde{\Eventually_I \varphi}$, there exists $\tau\in t+I$ such that $(\boldsymbol{x},\tau)\models \tilde{\varphi}$. By IH, for the same $\tau$, $(\boldsymbol{x}+\boldsymbol{e},\tau)\models \varphi$, hence $(\boldsymbol{x}+\boldsymbol{e},t)\models \Eventually_I \varphi$.

        \item Until: If $(\boldsymbol{x},t)\models \widetilde{\varphi\Until_I \psi}$, there exists $\tau\in t+I$ such that $(\boldsymbol{x},\tau)\models \tilde{\psi}$ and, for all $\tau'\in[t,\tau]$, $(\boldsymbol{x},\tau')\models \tilde{\varphi}$. By IH, $(\boldsymbol{x}+\boldsymbol{e},\tau)\models \psi$ and $(\boldsymbol{x}+\boldsymbol{e},\tau')\models \varphi$ for all $\tau'\in[t,\tau]$. Hence $(\boldsymbol{x}+\boldsymbol{e},t)\models \varphi\Until_I \psi$.
        \end{itemize}
    }
\end{enumerate}

By induction, \eqref{eq: erosion-strong} holds for all subformulas and thus for $\varphi$. 
Finally, using $\prob{\boldsymbol{e}\in\boldsymbol{E}}\ge 1-\delta$ from Definition~\ref{def: PRT}, we obtain
$
(\boldsymbol{x},0)\models \tilde{\varphi}\ \Rightarrow\ \prob{(\boldsymbol{X},0)\models \varphi}\ \ge\ 1-\delta.
$
\end{proof}

At a high level, the STL erosion mechanism separates \emph{randomness} from \emph{logic}: 
the PRT as Definition~\ref{def: PRT} bounds the fluctuation $e_t=X_t-x_t$ in a time-varying tube $E_t$, while predicate erosion (Proposition~\ref{prop: predicate erosion}) shrinks each predicate set $\mathcal{C}_i$ to its robust counterpart $\mathcal{C}_i\ominus E_t$ so that satisfaction is preserved under any $e_t\in E_t$. By compositionality of STL semantics, this predicate-level robustness lifts to the full formula (Theorem~\ref{thm: stl erosion}), yielding the guarantee
$(\boldsymbol{x},0)\models \tilde{\varphi} \Rightarrow \prob{(\boldsymbol{X},0)\models \varphi}\ge 1-\delta$.

\subsection{Deterministic Surrogate of Stochastic STLTO}

The predicate erosion strategy indicates that the chance constraint in \eqref{eq: stoch_STLTO} can be enforced conservatively by requiring the associated \emph{deterministic} trajectory $\boldsymbol{x}$ to satisfy the eroded formula $\tilde{\varphi}$ under the deterministic dynamics. This observation leads to a deterministic surrogate of the stochastic STL trajectory optimization problem, obtained by replacing the chance constraint with $(\boldsymbol{x},0)\models\tilde{\varphi}$:
\begin{subequations}\label{eq: det_STLTO}
\begin{align}
\min_{\boldsymbol{x},\,\boldsymbol{u}} \quad 
& J_d(\boldsymbol{x},\boldsymbol{u}) \;=\; \int_{0}^{T} \mathcal{L}_t(x_t,u_t)\,\mathrm{d}t \;+\; \Phi_T(x_T) \\
\text{s.t.}\quad 
& \dot{x}_t \;=\; f_t(x_t,u_t), \\
& \boldsymbol{x}(0)=x_0, \\
& (\boldsymbol{x},0)\models \tilde{\varphi}, \label{eq: det stl constraint}\\
& u_t\in\mathcal{U}, \qquad \forall t\in[0,T].
\end{align}
\end{subequations}

Unlike the intractable stochastic formulation~\eqref{eq: stoch_STLTO}, the surrogate problem~\eqref{eq: det_STLTO} is deterministic and can be solved using standard trajectory optimization techniques~\cite{kelly2017introduction,gilpin2020smooth,han2025exactsmoothreformulationstrajectory}. Let $\{x_t^{*},u_t^{*}\}_{t\in[0,T]}$ denote an optimal solution to~\eqref{eq: det_STLTO}. However, applying the feedforward control $u_t^{*}$ to the stochastic system~\eqref{eq: stochastic_system} in a purely open-loop manner allows stochastic deviations to accumulate. This results in an excessively large PRT, which requires overly aggressive predicate erosion and renders the deterministic surrogate~\eqref{eq: det_STLTO} infeasible. To maintain a tight PRT and preserve feasibility, it is essential to employ a closed-loop feedback controller. Specifically, we augment the nominal control with a state-feedback tracking controller that restricts the deviation of $X_t$ from its nominal counterpart $x_t^{*}$. The composite control input applied to the stochastic system is thus formulated as:
\begin{equation} \label{eq: ut=K+u*}
\textstyle u_t= K_t(X_t;x_t^{*}, u_t^{*}) + u_t^{*},
\end{equation}
where $K_t(\cdot;x_t^{*}, u_t^{*}): \real^n\to\real^p$ is a tracking controller satisfying $K_t(x_t^{*};x_t^{*}, u_t^{*})=0$. The design of $K_t(\cdot;x_t^{*}, u_t^{*})$ is detailed in Section~\ref{sec: feedback control}.

Algorithmically, this establishes a two-stage synthesis framework. In the offline phase, we solve~\eqref{eq: det_STLTO} to obtain the nominal trajectory and feedforward control $\{x_t^{*},u_t^{*}\}_{t\in[0,T]}$. During the online execution phase, the actual stochastic system~\eqref{eq: stochastic_system} is driven by the composite policy~\eqref{eq: ut=K+u*}. Crucially, as we formally establish in Theorem~\ref{thm:det-surrogate}, $X_t$ and $x_t^{*}$ remain valid associated trajectories under this closed-loop scheme. Consequently, the predicate erosion guarantees hold, ensuring that the online stochastic system satisfies the original STL specification with the desired high probability.

\begin{thm}[Deterministic surrogate of stochastic STLTO]\label{thm:det-surrogate}
Consider the stochastic system~\eqref{eq: stochastic_system} and its associated deterministic system~\eqref{eq: deterministic_system} under Assumption~\ref{as: boundness}. 
Fix $\delta\in(0,1)$ and let $\boldsymbol{E}=\{E_t\}_{t\in[0,T]}$ be a $1-\delta$ PRT as in Definition~\ref{def: PRT}. 
Construct the eroded STL formula $\tilde{\varphi}$ as in Theorem \ref{thm: stl erosion}. 
Let $\{x_t^{*},u_t^{*}\}_{t\in[0,T]}$ be a feasible solution of \eqref{eq: det_STLTO}.
Then any $u_t= K_t(X_t;x_t^{*}, u_t^{*}) + u_t^{*}$ with $K_t(x_t^{*};x_t^{*}, u_t^{*})=0$ satisfies the chance constraint in the stochastic problem~\eqref{eq: stoch_STLTO}, i.e.,
\[
(\boldsymbol{x^{*}},0)\models \tilde{\varphi}
\quad\Longrightarrow\quad
\prob{(\boldsymbol{X},0)\models \varphi}\ge 1-\delta.
\]
Consequently, any feasible solution of \eqref{eq: det_STLTO} yields a control policy $\boldsymbol{u}$ as \eqref{eq: ut=K+u*} that is feasible for the stochastic STLTO~\eqref{eq: stoch_STLTO}.
\end{thm}

\begin{proof}[Proof]
Define $f_{t}^{cl}(x,u;x_t^{*}, u_t^{*})=f_t(x,K_t(x;x_t^{*}, u_t^{*})+u)$. The solution $\{x_t^{*},u_t^{*}\}_{t\in[0,T]}$ of \eqref{eq: det_STLTO} and the stochastic trajectory $X_t$ controlled by $u_t$ as \eqref{eq: ut=K+u*} can thus be expressed as:
\begin{equation} \label{eq: cl_system}
    \begin{split}
        \dot{x}^{*}_t&=f_{t}^{cl}(x_t^{*},u_t^{*};x_t^{*},u_t^{*}) \\
        \dX_t&=f_{t}^{cl}(X_t,u_t^{*};x_t^{*},u_t^{*})\dt+g_t(X_t)\dW_t
    \end{split}
\end{equation}
Since $x_0^{*}=X_0$ and the same $u_t^{*}$ is applied, we know $x_t^{*}$ and $X_t$ are associated trajectories with respect to the closed-loop transition function $f_{t}^{cl}$. Therefore, by Theorem~\ref{thm: stl erosion}, $(\boldsymbol{x^{*}},0)\models\tilde{\varphi}$ implies $(\boldsymbol{x^{*}}+\boldsymbol{e},0)\models\varphi$, which establishes the feasibility of the chance constraint  $\prob{(\boldsymbol{X},0)\models\varphi}\ge 1-\delta$. Consequently, $\{X_t, u_t\}$ forms a feasible solution of \eqref{eq: stoch_STLTO}.
\end{proof}

The feasibility of the predicate-erosion strategy depends on the tightness of the PRT radius profile $r_{\delta,t}$ as in Definition \ref{def: PRT}. If $r_{\delta,t}$ is overly conservative, then the PRT $E_t$ can be so large that makes $\mathcal{C}^i\ominus E_t$ very small or even empty, leading to an infeasible $\tilde{\varphi}$ in \eqref{eq: det_STLTO}. In the next section, we develop a tight PRT.

\newcommand{\upsig}{\overline{\sigma}}
\newcommand{\upr}{\overline{r}}
\section{A Tight Bound on PRT}
\label{sec: prob_bound}
In this section, we develop a tight bound $r_{\delta,t}$ on the PRT defined in Definition \ref{def: PRT}. 

\subsection{Contraction Theory}
Under the composite control law~\eqref{eq: ut=K+u*}, the feedback term $K(X_t;x_t^{*})$ suppresses the deviation between the stochastic state $X_t$ and the nominal state $x_t^{*}$. This closed-loop regulation tightens the resulting PRT. More specifically, the tube radius $r_{\delta,t}$, which establishes the probabilistic bound on the stochastic deviations, is fundamentally governed by the contraction properties of the closed-loop dynamics. To formalize and analyze this behavior for nonlinear systems, we leverage \textit{contraction theory}.

\begin{definition} [Contracting System] \label{def: contraction}
    A deterministic system formulated as \eqref{eq: deterministic_system} is said to be $c_t$-\textit{contracting} if for any two trajectories $x_t$, $y_t$ of \eqref{eq: deterministic_system}, $\exists c_t\in\R$ and $M_t\in \mathbb{S}_{++}^n$ such that:
        \begin{equation} \label{eq: contracting M}
    \frac{\dd}{\dt}\|x_t-y_t\|_{M_t}^2\leq 2c_t\|x_t-y_t\|_{M_t}^2
\end{equation}
holds at time $t\in[0,T]$. Here, $c_t$ is called the \textit{contraction rate} and $M_t$ is the called \textit{contraction metric}. In particular, the system is said to be strongly contracting if $c_t<0$ holds for all $t$.
\end{definition}

\begin{remark}
    While the most general, geometric formulation of contraction theory is defined on Riemannian manifolds~\cite{tsukamoto2021contraction}, we adopt a different formulation in Definition~\ref{def: contraction} that avoids involving Riemannian geometry. This setting is fully sufficient for the scope of this paper and avoids unnecessary complexity. For an extended theoretical analysis based on Riemannian manifolds, we refer the reader to~\cite{liu2026concentration}.
\end{remark}

The contracting systems have many intriguing properties. One property is that a contracting system with arbitrary contraction rates $c_t$ has a counterpart with contraction rate exactly equal to $0$. This property allows analysis on the special case $c_t\equiv0$ to be smoothly generalized to general cases, and will be useful in Section \ref{sec: prob_bound}.
\begin{lemma} \label{lemma: ct -> c=0}
    Let $x_t$ be a deterministic trajectory of \eqref{eq: deterministic_system}. Suppose that it is contracting with the rate $c_t$ and the metric $M_t$ as defined in Definition \ref{def: contraction}. Then for the trajectory $y_t=e^{\int_0^t c_\tau\dtau}x_t$, its dynamics is contracting with the rate $c_t^y\equiv0$ and the metric $M_t^y=M_t$.
\end{lemma}
\begin{proof}
    See \cite{szy2024TAC}.
\end{proof}

For the closed-loop system~\eqref{eq: cl_system} under the control policy~\eqref{eq: ut=K+u*}, a contraction rate of $c_t \ge 0$ implies that the stochastic trajectory $X_t$ may locally diverge from the nominal trajectory $x_t^*$. Such divergence fundamentally undermines the objective of the tracking controller. Consequently, we require that the feedback policy $K(\cdot;x_t^{*})$ is appropriately designed to render the closed-loop dynamics $f_{cl}$ \emph{strongly contracting}.
\begin{assumption} \label{as: contraction}
    The closed-loop system \eqref{eq: cl_system} is strongly contracting with $c_t<0$ under the control policy \eqref{eq: ut=K+u*}.
\end{assumption}
We introduce the feedback controller design in Section \ref{sec: feedback control} to ensure that Assumption \ref{as: contraction} is satisfied.

Existing works based on the Incremental Stability Analysis \cite{pham2009contraction,dani2014observer,tsukamoto2020robust} have shown the strong relationship between the bound on the PRT and the contraction condition of the system. However, the bounds established in these existing works are overly conservative when the probability level $1-\delta$ is very high (e.g., 99.99\%), and are not on the trajectory level, thus cannot be utilized in the STL erosion strategy. Next, we develop a trajectory-level bound on the PRT that is tight even when $\delta$ is very small.

\subsection{Tight Bound on Probabilistic Reachable Tube} \label{sec: tight_bound}
Based on the contraction condition, we use a martingale-based method to derive a tight bound $r_{\delta,t}$ on the PRT in Definition \ref{def: PRT}.

The family of martingales can be treated as stochastic analogs of barrier functions\cite{steinhardt2012finite}, and is a common approach to bound the probability of a stochastic trajectory staying in a region \cite{BEUTLER1973464}.
Based on the $c$-martingale \cite{steinhardt2012finite} widely used in robotics, we introduce a novel generalization dubbed Affine martingale (AM).

\begin{definition} [Affine Martingale] \label{def: CT AM}
   For a continuous stochastic process $Y_t,~ t\in[0,T]$, a nonnegative differentiable function $B(Y,t):\R^n\times[0,T]\to\R_{\geq0}$ is said to be an affine martingale (AM) of $Y_t$ if there exist $a_t\in\R,b_t\in\R_{\geq0}$ such that for all $t\leq T$ and as $\dt\to 0$: $\frac{\expect{B(Y_{t+\dt},t+\dt)|Y_t}-B(Y_t,t)}{\dt}\leq a_tB(Y_t,t)+b_t.$
\end{definition}

When $a_t\equiv0$, the AM reduces to the classical $c$-martingale. This type of martingale first appeared in \cite[Chapter 3]{1967stochastic}, and Definition \ref{def: CT AM} slightly generalizes it to time-varying coefficients. Similar with existing martingale-based methods, one can utilize supermartingale theory to bound the probability of $\{Y_t\}_{t\in[0,T]}$ staying in certain sets. This is formalized in the following lemma and its proof can be found in \cite[Lemma 4.1]{liu2025safetyverificationnonlinearstochastic}.
\begin{lemma}\label{lemma: CT-AM}
    Consider a continuous-time stochastic trajectory $Y_t,~ t\in[0,T]$. Let $B(Y,t)$ be an AM of $Y_t$ with coefficients $a_t, b_t\geq0$.
    Define $\widetilde{B}(Y_t,t):=B(Y_t,t)\xi_t+\int_t^T b_\tau\xi_\tau \dd\tau$,
    where $\xi_t:=e^{\int_t^Ta_\tau\dd\tau}$. Then given any $\overline{B}>0$ and the set $\mathcal{Y}_t:=\{Y_t:~\widetilde{B}(Y_t,t)\leq \overline{B}\}$, it holds that $ \prob{Y_t\in\mathcal{Y}_t, \forall t\leq T}\geq 1-\frac{B(v_0,0)\xi_0+\int_0^Tb_\tau\psi_\tau\dd\tau}{\overline{B}}.$
\end{lemma}

One may wonder how to construct an AM that can probabilistically bound the fluctuation of the stochastic trajectory. The answer lies in a novel function named Averaged Moment Generating Function (AMGF), which was initially proposed in \cite{altschuler2022concentration} and fully exploited in \cite{szy2024TAC,liu2025safetyverificationnonlinearstochastic}. The definition of AMGF is as follows:
\begin{definition} \label{def: AMGF}
    Given $X\in\R^n$, the averaged moment generating function (AMGF) is defined as 
    \begin{equation}
        \textstyle \mbE_X\Phi(X)=\mbE_X\mathbb{E}_{\ell\in\mS^{n-1}}e^{\lambda\innerp{\ell,X}},
    \end{equation}
    where $\mS^{n-1}=\{x\in\R^n: \|x\|=1\}$ is the unit sphere in the $\R^n$ space, and $\Phi(X)=\mathbb{E}_{\ell\in\mS^{n-1}}e^{\lambda\innerp{\ell,X}}$ is defined as the kernel of the AMGF. Moreover, given $M\in\mathbb{S}^{n}_{++}$, the weighted version of $\Phi(X)$ is defined as $\Phi_M(X)=\Phi(M^{1/2}X)$.
\end{definition}

The kernel of weighted AMGF has many intriguing properties. The following lemma lists some of them that are useful in this paper. Its proof directly follows \cite{liu2025safetyverificationnonlinearstochastic}.
\begin{lemma} \label{lemma: M-AMGF}
    Consider the function $\Phi_M(X)$ in Definition \ref{def: AMGF}, where $X\in\R^n$ and $M\in\mathbb{S}^{n}_{++}$, then it holds that: 
    \begin{enumerate}
    \item The value of $\Phi_M(X)$ only depends on $\|X\|_M$.
    \item $\Phi_{M_1}(X_1)\leq\Phi_{M_2}(X_2)$ if $\|X_1\|_{M_1}\leq \|X_2\|_{M_2}$.
    \item  Given any $\varepsilon\in(0,1)$, $\Phi_M(X)\geq (1-\varepsilon^2)^{\frac{n}{2}}e^{\varepsilon \|\lambda X\|_M}.$
\end{enumerate}
\end{lemma}

Equipped with AMGF, we are able to construct AM for nonlinear stochastic systems, and derive a tight bound for the PRT. The bound is demonstrated in the following lemma.

\begin{thm} 
    \label{thm: prob_bound}
    Consider a trajectory $X_t$ of the closed-loop stochastic system \eqref{eq: cl_system} with the control policy \eqref{eq: ut=K+u*}, 
    and its associated nominal trajectory $\{x_t^*,u_t^*\}_{t\in[0,T]}$ as a feasible solution of \eqref{eq: det_STLTO}. 
    Suppose that Assumptions \ref{as: boundness} and \ref{as: contraction} hold for the closed-loop system \eqref{eq: cl_system}.
    Given the terminal time $T$, $\delta\in(0,1)$, $\varepsilon\in(0,1)$ and $\Delta t>0$, define $\overline{\sigma}_t=\sqrt{\|M_t\|}\sigma$ where $\sigma$ is as Assumption \ref{as: boundness}, $k=\lceil \frac{t}{\Delta t}\rceil$, $\psi_t=\int_0^t c_\tau \dtau$ where $c_\tau$ is the contraction rate introduced in Definition~\ref{def: contraction}, $\Psi_t=\int_0^t \upsig_\tau^2e^{-2\psi_\tau}\dtau$, $\Psi_{t}^{\Delta t}=\int_{k\Delta t}^{(k+1)\Delta t} \upsig_\tau^2e^{-2\psi_\tau}\dtau$, $\varepsilon_1=\frac{\log(\frac{1}{1-\varepsilon^2})}{\varepsilon^2},~ \varepsilon_2=\frac{2}{\varepsilon^2}$, and:
     \begin{equation}\label{eq: r traj}
        r_{\delta,t}=
        (\sqrt{e^{2\psi_t}\Psi_t}+\sqrt{\Psi_{t}^{\Delta t}})\sqrt{\varepsilon_1n+\varepsilon_2\log\frac{2T}{\delta \Delta t}}.
    \end{equation}
     Then it holds that
    \begin{equation}
        \prob{\|X_t-x_t\|_{M_t}\leq r_{\delta,t}, \forall t\leq T}\geq 1-\delta.
    \end{equation}
\end{thm} 

\begin{proof}
We start with the special case where $c_t=0$. For associated trajectories $X_t$ and $x_t$, recall $e_t=X_t-x_t$, define $\hat{e}_t=M_t^{1/2}e_t$, $\beta_t=f_{cl}(X_t,u_t^{*},t;x_t^{*})-f_{cl}(x_t^{*},u_t^{*},t;x_t^{*})$, $V_t=\Phi_{M_t}(e_t)$ and $\mA(V_t)=\frac{\mbE(V_{t+\dt}|V_t)-V_t}{\dt}$. Obviously, $\dd e_t=\beta_t\dt+g_t(X_t)\dW_t$. Then, by Ito's Lemma, the evolution of $V_t$ can be extracted as
\begin{equation} \label{eq: dv=Z1+Z2}
    \begin{split}
        \textstyle
    &\mA(V_t)= Z_1+Z_2,~ \text{where:} \\
    &Z_1=\innerp{\frac{\partial V_t}{\partial \hat{S}_t},\frac{\dd M_t^{1/2}}{\dt}e_t+M^{1/2}_t\beta_t}  \\
    &Z_2=\frac{1}{2}\innerp{\nabla^2V_t,g_t(X_t)g_t(X_t)^\top}.
    \end{split}
\end{equation} 

To analyze $Z_1$, consider a trajectory $\bar{x}_\tau$, $\tau\in[0,T]$ such that $\dot{\bar{x}}_{\tau}=f(\bar{x}_\tau,u_\tau^{*},\tau)$ with $\bar{x}_t=X_t$. By Ito's Lemma, it holds that $Z_1=\frac{\dd}{\dtau}\Phi_{M_\tau}(\bar{x}_\tau-x_\tau)|_{\tau=t}$. Notice that $\Phi_{M_\tau}(\bar{x}_\tau-x_\tau)$ is monotonically increasing with $\|\bar{x}_\tau-x_\tau\|_{M_\tau}$, which means $Z_1$ should have the same sign ($\pm$) as $\frac{\dd}{\dtau}\|\bar{x}_\tau-x_\tau\|_{M_\tau}^2|_{\tau=t}$. When $c_t=0$, the contraction condition implies that $\frac{\dd}{\dtau}\|\bar{x}_\tau-x_\tau\|_{M_\tau}^2\leq 2c_\tau\|\bar{x}_\tau-x_\tau\|_{M_\tau}^2=0$. Therefore,  $Z_1\leq0$.

$Z_2$ can be unfolded and bounded as follows:
\begin{equation*}\label{eq: part 2 FPK}
    \begin{split}
        Z_2=&\tfrac{1}{2}\innerp{\nabla^2V_t,g_t(X_t)g_t(X_t)^\top} \\
        =&\tfrac{1}{2} \expectw{\ell\sim\mS^{n-1}}{\innerp{\lambda^2e^{\lambda\innerp{\ell ,\bar{S}_t}}M_{t}\ell\ell^{\top},g_t(X_t)g_t(X_t)^{\top}}} \\
        \leq& \tfrac{1}{2}\expectw{\ell\sim\mS^{n-1}}{\lambda^2e^{\lambda\innerp{\ell,\bar{S}_t}}\tr{\ell\ell^{\top}}\,\|M_{t}g_t(X_t)g_t(X_t)^{\top}\|} \\
        \leq& \tfrac{\lambda^2\|M_t\|\sigma^2}{2}V_t \leq \tfrac{\lambda^2\upsig_t^2}{2}V_t.
    \end{split}
\end{equation*}

Combining the bounds on $Z_1$ and $Z_2$, we get
\begin{equation} \label{eq: dV<=aV+b}
\begin{split}
    \mA(V_t) \leq  \tfrac{\lambda^2\upsig_t^2}{2}V_t,
\end{split}
\end{equation}
which implies that $V_t=\Phi_{M_t}(X_t-x_t)$ is an affine martingale with $a_t=\frac{\lambda^2\upsig_t^2}{2}$ and $b_t\equiv0$. Therefore, for every $\upr>0$ and any $\eta\in\mS^{n-1}$, we get
\begin{equation}\label{eq: CT prob |v_t|<r_lam}
    \begin{split}
        &\prob{\|X_t-x_t\|_{M_t}\leq \upr, \forall t\leq T} \\
        =&  \prob{\|X_t-x_t\|_{M_t}\leq \|\upr\eta\|,~ \forall t\leq T}\\
        =&\prob{\Phi_{M_t}(X_t-x_t)\leq \amgf{\upr\eta}, \forall t\leq T} ~~~~ [\text{\textbf{Lemma \ref{lemma: M-AMGF}}}] \\
        =& \prob{e^{\frac{\lambda^2\int_t^T\upsig_\tau^2\dtau}{2}}\mathcal{E}(l^*_t)\leq e^{\frac{\lambda^2\int_t^T\upsig_\tau^2\dtau}{2}}\amgf{\upr\eta}, \forall t\leq T} \\
        \geq& \prob{e^{\frac{\lambda^2\int_t^T\upsig_\tau^2\dtau}{2}}\mathcal{E}(l^*_t)\leq \amgf{\upr\eta}, \forall t\leq T} \\
        \geq &1-\frac{e^{\frac{\lambda^2\int_0^T\upsig_\tau^2\dtau}{2}}}{\amgf{\upr\eta}},\quad \forall \eta\in\mS^{n-1} ~~~~[\textbf{Lemma \ref{lemma: CT-AM}}]\\
        \geq & 1- (1-\varepsilon^2)^{-\frac{n}{2}}\exp\left(\tfrac{\lambda^2\int_0^T\upsig_\tau^2\dtau}{2}-{\varepsilon\lambda \upr}\right) [\text{\textbf{Lemma \ref{lemma: M-AMGF}}}].
    \end{split}
\end{equation}
Minimizing the last line of \eqref{eq: CT prob |v_t|<r_lam} over $\lambda$, we get $\lambda^{*}=\frac{\varepsilon r}{\int_0^T\upsig_\tau^2\dtau}$. By Plugging $\lambda=\lambda^{*}$ into \eqref{eq: CT prob |v_t|<r_lam} and setting $$\upr=\sqrt{\tfrac{2\int_0^T\upsig_\tau^2\dtau}{\varepsilon^2}(\tfrac{n}{2}\log\left(\tfrac{1}{1-\varepsilon^2}\right)+\log(1/\delta))},$$
we arrive at the result of Theorem \ref{thm: prob_bound} in the case that $c\to0$. 

Next, to generalize from $c_t=0$ to any $c_t\in\R$, define $\tX_t=e^{-\psi_t}X_t$ and $\tx_t=e^{-\psi_t}x_t$. Lemma \ref{lemma: ct -> c=0} implies that the dynamics of $\tx_t$ has the contraction rate $\tilde{c}_t=0$, and the diffusion term of the dynamics of $\tX_t$ satisfies $\|e^{-\psi_t}g_t(X_t)e^{-\psi_t}g_t(X_t)^{\top}\|\leq e^{-2\psi_t}\sigma^2:= \tilde{\sigma}_t^2$. Therefore, following \eqref{eq: dv=Z1+Z2}-\eqref{eq: dV<=aV+b}, we know $\tilde{V}_t=\Phi_{M_t}(\tX_t-\tx_t)$ is an AM with $a_t=\frac{\lambda^2\|M_t\|\tilde{\sigma}_t^2}{2}$ and $b_t\equiv0$. Then, following \eqref{eq: CT prob |v_t|<r_lam} and the steps afterwards, we get $\prob{\|X_t-x_t\|_{M_t}\leq r_{\delta,t}, \forall t\leq T}=\prob{\|\tX_t-\tx_t\|_{M_t}\leq \tilde{r}_{\delta,t}, \forall t\leq T}\geq 1-\delta$, where
\begin{equation}\label{eq: CT r c>0}
\begin{split}
r_{\delta,t}=e^{\psi_t}\tilde{r}_{\delta,t}=\sqrt{e^{2\psi_t}\Psi_T(\varepsilon_1n+\varepsilon_2\log(1/\delta))}.
\end{split}
\end{equation}
This completes the first part of the proof.

Although \eqref{eq: CT r c>0} is mathematically correct for any $c_t\in\R$, we point out that this result can be overly conservative when $c_t<0$ and $t\ll T$ \cite{liu2026concentration}. 
For the rest of the proof, we improve the tightness of the bound $r_{\delta,t}$ for the case where $c_t<0$. Since the conservativeness appears when $T\gg t$, our strategy is to split $[0,T]$ into small intervals with length $\Delta t$, so that $t$ is close to $\lceil \frac{t}{\Delta t}\rceil\Delta t$, which is the endpoint of the interval that contains $t$. 

Define $N=T/\Delta t$.\footnote{We choose $\Delta t$ so that $N=T/\Delta t$ is an integer for convenience, but the same conclusion holds for arbitrary $\Delta t$.} 
For any $t\in(k\Delta t, (k+1)\Delta t)$, $k=0,\dots,N-1$, define a trajectory $y_t^{(k)}$ that satisfies
\begin{equation}
    \begin{split}
        \dot{y}_t^{(k)} = f(y_t^{(k)},u_t,t), \,\,\, y_{k\Delta t}^{(k)} = X_{k\Delta t}.
    \end{split}
\end{equation}
Note that $X_t$ and $y_t^{(k)}$ are associated trajectories over the time horizon $(k\Delta t, (k+1)\Delta t)$.
Let $\upr^\Delta=\sqrt{\Psi_{t}^{\Delta t}(\varepsilon_1n+\varepsilon_2\log\frac{2N}{\delta})}$
and $\tilde{r}^{\Delta}=e^{\int_{k\Delta t}^t c_\tau\dtau}\upr^\Delta$, then $\tilde{r}^{\Delta}\leq\upr^\Delta$ and it holds that
\begin{equation} \label{eq: t in k,k+1 dt}
    \begin{split}
        &\prob{\|X_t-y_t^{(k)}\|_{M_t}\leq \upr^\Delta,~ \forall t\in(k\Delta t, (k+1)\Delta t)} \\
        \geq &\prob{\|X_t-y_t^{(k)}\|_{M_t}\leq \tilde{r}^{\Delta},~ \forall t\in(k\Delta t, (k+1)\Delta t)} \\
        \geq &1-\frac{\delta}{2N},
    \end{split}
\end{equation}
 where the second ``$\geq$'' directly follows Theorem \ref{thm: prob_bound} by setting the time period as $\Delta t$ and the initial time as $k\Delta t$.

For $t=k\Delta t$, $k=1,\dots, N$, we need to derive the bound $r_{k\Delta t}$ such that 
\begin{equation}\label{eq: t=kdt}
\begin{split}
    \mathbb{P}\left(\|X_{k\Delta t}-x_{k\Delta t}\|_{M_{k\Delta t}}\leq r_{k\Delta t}\right)\geq 1-\frac{\delta}{2N}
\end{split}
\end{equation}
holds at $k\Delta t$. Such a \textit{single-time} probabilistic bound has been well-studied in \cite{szy2024TAC,liu2026concentration}, following which $r_{k\Delta t}$ can be determined as
 \begin{equation} 
     r_{k\Delta t}= \sqrt{e^{2\psi_{k\Delta t}}\Psi_{k\Delta t}(\varepsilon_1n+\varepsilon_2\log\frac{2N}{\delta})}.
 \end{equation}

Finally, we union the bounds \eqref{eq: t in k,k+1 dt} and \eqref{eq: t=kdt} to arrive the final result. Define
\begin{equation*}
    \begin{split}
       \textstyle &E_k^{(1)}:~ \|X_{k\Delta t}-x_{k\Delta t}\|\leq r_{k\Delta t} \\
        &E_k^{(2)}:~ \|X_t-y_t^{(k)}\|\leq r^\Delta,~ \forall t\in(k\Delta t, (k+1)\Delta t).
    \end{split}
\end{equation*}
By \eqref{eq: t in k,k+1 dt},\eqref{eq: t=kdt} and the union-bound inequality, 
$$\textstyle \prob{\left(\bigcap_{k=1}^N (E_k^{(1)} \bigcap E_k^{(2)})\right)}\geq 1-2\sum_{k=1}^{N} \frac{\delta}{2N} =1-\delta.$$ 
Notice that when these events happen, it holds that 
\begin{equation}\label{eq: CT improved bound}
    \begin{split}
        \|X_t-x_t\|\leq &  \|x_t-y_t^{(k)}\|+\|X_t-y_t^{(k)}\| \\
        \leq & \|X_{k\Delta t}-x_{k\Delta t}\|+\|X_t-y_t^{(k)}\| \\
        \leq & r_{k\Delta t}+\upr^\Delta \\
        \leq & (\sqrt{e^{2\psi_t}\Psi_t}+\sqrt{\Psi_{t}^{\Delta t}})\sqrt{\varepsilon_1n+\varepsilon_2\log\frac{2T}{\delta \Delta t}},
    \end{split}
\end{equation}

where the second ``$\leq$'' is because by integrating the contraction condition:
\begin{equation}\label{eq: xt'-xt}
\begin{split}
    &\|x_t-y_t^{(k)}\|_{M_t}\leq e^{\int_{k\Delta t}^t c_\tau\dtau}\|x_{k\Delta t}-y_{k\Delta t}^{(k)}\|_{M_{k\Delta t}} \\
    \leq &\|x_{k\Delta t}-y_{k\Delta t}^{(k)}\|_{M_{k\Delta t}}=\|x_{k\Delta t}-X_{k\Delta t}\|_{M_{k\Delta t}}.
\end{split}
\end{equation}

Combining \eqref{eq: CT improved bound} and \eqref{eq: CT r c>0}, we conclude that
 \begin{equation}
\mathbb{P}\left(\|X_t-x_t\|_{M_t}\leq r_{\delta,t}, ~\forall t\leq T\right)\geq 1-\delta,
   \end{equation} 
where $r_{\delta,t}$ satisfies \eqref{eq: r traj}. This completes the proof.
\end{proof}

\begin{remark}
    When $c_t\equiv c$ is time-invariant, the $r_{\delta,t}$ in \eqref{eq: r traj} reduces to the following result:
    \begin{equation}
        r_{\delta,t}
=
\frac{
    \bar\sigma_t\!\left(
        \sqrt{1 - e^{2ct}}
        +
        \sqrt{e^{-2c\Delta t} - 1}
    \right)
}{
    \sqrt{-2c}
}
\sqrt{
    \varepsilon_1 n + \varepsilon_2 \log\frac{2T}{\delta \Delta t}
}.
    \end{equation}
\end{remark}

Given the terminal time $T$ and the probability level $\delta$, the bound $r_{\delta,t}$ derived in Theorem \ref{thm: prob_bound} has an $\mO(\sqrt{\log T/\delta})$ dependence, which is the tightest case one can expect for Ito's stochastic systems \cite{liu2025safetyverificationnonlinearstochastic}. Therefore, the satisfaction of the constraint (\ref{eq: det_STLTO}d) is insensitive with $\delta$ and $T$ when choosing $M_t$ as the contraction metric and $r_{\delta,t}$ as \eqref{eq: r traj}. Also, \cite{szy2024TAC} has shown that the dependence of $r_{\delta,t}$ on $\sigma$ and $n$ has reached the tightest case under Assumption \ref{as: boundness}. Therefore, the core of tightening the bound $r_{\delta,t}$ derived in Theorem \ref{thm: prob_bound} is to decrease the value of $c_t$ and $\|M_t\|$ of the closed-loop system $f_{cl}$, which depends on the design of the tracking controller $K_t(\cdot;x_t^*, u_t^*)$.
 
\section{Feedback Motion Planning}
\label{sec: feedback control}

In this section, we focus on the design of feedback tracking controllers for the nominal trajectory generated by \eqref{eq: det_STLTO}. The role of the feedback layer is twofold: it stabilizes the stochastic trajectory around the nominal plan, and it induces the contraction metric and contraction rate needed by Theorem~\ref{thm: prob_bound}. All numerical optimization and solver implementations are deferred to Section~\ref{sec:numerical_optimization}.

Let $(\boldsymbol{x}^*,\boldsymbol{u}^*)$ be a nominal state-control trajectory pair solving the deterministic surrogate problem \eqref{eq: det_STLTO}. We seek a time-varying tracking law of the form
\begin{equation}
\label{eq:tracking_law_general}
u_t = u_t^* + K_t(X_t;x_t^*,u_t^*),
\end{equation}
where $K_t(x_t^*;x_t^*,u_t^*)=0$. Under this policy, the stochastic trajectory is regulated toward the nominal plan, yielding a closed-loop system that is strongly contracting in the sense of Definition~\ref{def: contraction}.

\subsection{Trajectory Tracking via Time-Varying LQR}
\label{sec: tvlqr_tracking}

We first consider a Time-Varying Linear Quadratic Regulator (TVLQR) constructed along the nominal trajectory $(\boldsymbol{x}^*,\boldsymbol{u}^*)$. Linearizing the deterministic dynamics \eqref{eq: deterministic_system} along the nominal plan gives the time-varying Jacobians
\begin{equation*}
A_t = \frac{\partial f}{\partial x}\bigg|_{x_t^*,u_t^*,t},
\qquad
B_t = \frac{\partial f}{\partial u}\bigg|_{x_t^*,u_t^*,t}.
\end{equation*}
Letting $\tilde{x}_t = x_t - x_t^*$ and $\tilde{u}_t = u_t - u_t^*$ denote the tracking error and control deviation, respectively, the local error dynamics are approximated by
\begin{equation*}
\dot{\tilde{x}}_t = A_t \tilde{x}_t + B_t \tilde{u}_t.
\end{equation*}

Given positive semidefinite state penalty $Q_t \succeq 0$, positive definite input penalty $R_t \succ 0$, and terminal penalty $Q_f \succ 0$, the TVLQR gain is obtained from the continuous-time differential Riccati equation
\begin{equation}
    -\dot{S}_t = S_t A_t + A_t^\top S_t - S_t B_t R_t^{-1} B_t^\top S_t + Q_t,
    \quad S_T = Q_f.
    \label{eq: cdre}
\end{equation}
The resulting feedback law is
\begin{equation*}
\tilde{u}_t = -K_t \tilde{x}_t,
\quad
K_t = R_t^{-1} B_t^\top S_t,
\end{equation*}
or equivalently,
\begin{equation*}
\label{eq:tvlqr_control_law}
u_t = u_t^* - K_t(X_t-x_t^*).
\end{equation*}

The closed-loop linearized dynamics are
\begin{equation*}
\dot{\tilde{x}}_t = A_{\mathrm{cl},t}\tilde{x}_t,
\qquad
A_{\mathrm{cl},t}=A_t-B_tK_t.
\end{equation*}
We use the Riccati solution itself as the contraction metric candidate, i.e., $M_t := S_t.$

To characterize the induced contraction rate, consider the quadratic energy
\begin{equation}
V_t = \tilde{x}_t^\top S_t \tilde{x}_t.
\end{equation}
Combining the closed-loop dynamics with \eqref{eq: cdre} yields
\begin{equation}
\begin{split}
\dot{V}_t
&= \tilde{x}_t^\top \big(\dot{S}_t + S_tA_{\mathrm{cl},t} + A_{\mathrm{cl},t}^\top S_t\big)\tilde{x}_t \\
&= -\tilde{x}_t^\top \big(Q_t + K_t^\top R_t K_t\big)\tilde{x}_t.
\end{split}
\end{equation}
To extract a contraction rate, we compare $\dot V_t$ with $V_t$. A contraction rate $c_t$ is valid if
\begin{equation}
\dot V_t \le 2c_t V_t
\end{equation}
for all tracking errors $\tilde{x}_t$. Equivalently,
\begin{equation}
-\tilde{x}_t^\top\big(Q_t+K_t^\top R_tK_t\big)\tilde{x}_t
\le
2c_t\,\tilde{x}_t^\top S_t\tilde{x}_t.
\end{equation}
Let $z=S_t^{1/2}\tilde{x}_t$, then for any $z_t$ the following inequality should hold

\begin{equation}
\frac{
z_t^\top S_t^{-1/2}\big(-Q_t-K_t^\top R_t K_t\big)S_t^{-1/2}z_t
}{
z_t^\top z_t
}
\le
2c_t
\end{equation}
Hence the local contraction rate can be taken as
\begin{equation}
    c_t =
    \frac{1}{2}
    \lambda_{\max}\!\Big(
    S_t^{-1/2}
    \big(-Q_t-K_t^\top R_t K_t\big)
    S_t^{-1/2}
    \Big),
    \label{eq: local_contraction}
\end{equation}
which is non-positive when $Q_t \succeq 0$ and $R_t \succ 0$. Since $S_t$ is continuous and positive definite on the compact interval $[0,T]$, there exist global matrices $M_1,M_2 \succ 0$ such that
\begin{equation}
M_1 \preceq S_t \preceq M_2, \qquad \forall t\in[0,T].
\end{equation}
These bounds can then be used in Theorem~\ref{thm: prob_bound}. The TVLQR construction is straightforward and computationally efficient, but it can be conservative due to the implicit metric induced by the Riccati solution.

\subsection{Time-Varying CCM Tracking Controller}
\label{sec:tvccm_tracking}

To reduce conservativeness beyond TVLQR tracking, we also consider a time-varying control contraction metric (TVCCM) construction. 
Linearizing along the nominal trajectory $(x_t^*,u_t^*)$ gives the differential dynamics
\begin{equation}
\delta \dot{x} = A_t \delta x + B_t \delta u,
\end{equation}
where
\begin{equation*}
A_t = \frac{\partial f}{\partial x}\bigg|_{x_t^*,u_t^*,t},
\qquad
B_t = \frac{\partial f}{\partial u}\bigg|_{x_t^*,u_t^*,t}.
\end{equation*}

The goal of TVCCM synthesis is to construct a differentiable dual metric $W_t \in \mathbb{S}_{++}^n$ and a differential feedback matrix $Y_t \in \mathbb{R}^{m\times n}$ such that the closed-loop differential dynamics are contracting with rate $c<0$. Defining
\begin{equation*}
M_t = W_t^{-1},
\qquad
K_t = Y_tW_t^{-1},
\end{equation*}
the continuous-time dual contraction condition~\cite{singh2017robust} is
\begin{equation}
\label{eq:continuous_tvccm_lmi}
-\dot{W}_t
+ A_tW_t + W_tA_t^\top
+ B_tY_t + Y_t^\top B_t^\top
\preceq -2cW_t.
\end{equation}

To improve the tightness of the resulting probabilistic tube, we further seek metrics that are well-conditioned and small in the task-relevant state projection. If the STL predicates are defined mainly in the position coordinates, then these coordinates are the most relevant for erosion (e.g., if the predicates consider goal regions or obstacle avoidance). We therefore introduce a common upper bound matrix $\bar{W}\succ0$ and scalar $\bar{\beta}>0$ such that
\begin{equation}
\label{eq:continuous_tvccm_metric_bound}
I \preceq W_t \preceq \bar{W},
\qquad
\bar{W}\preceq \bar{\beta}I,
\qquad \forall t\in[0,T],
\end{equation}
and penalize $\bar{\beta}$ together with the projected metric size
\begin{equation*}
\mathrm{tr}(P\bar{W}P^\top),
\end{equation*}
where $P$ selects the state coordinates appearing most prominently in the STL predicates, such as position states.

Once $(W_t,Y_t)$ is synthesized, the tracking controller is implemented as
\begin{equation*}
\label{eq:tvccm_control_law}
u_t = u_t^* + K_t(X_t-x_t^*),
\qquad
K_t=Y_tW_t^{-1}.
\end{equation*}
The induced contraction metric is then $M_t = W_t^{-1}$,
which can be used directly in the PRT bound of Section~\ref{sec: prob_bound}. Compared with TVLQR, TVCCM offers more freedom in shaping the metric and the contraction rate, and therefore can lead to tighter PRTs.

\subsection{Coupling Between Nominal Planning and Feedback Design}
\label{sec:planning_feedback_coupling}

As the readers might have noticed, there is an intrinsic coupling between nominal trajectory optimization and feedback design. The tightened STL constraint in \eqref{eq: det_STLTO} depends on the PRT radius profile $r_{\delta,t}$, since this radius determines the amount of predicate erosion. However, $r_{\delta,t}$ is itself determined by the contraction rate and contraction metric of the tracking controller through Theorem~\ref{thm: prob_bound}. Both TVLQR and trajectory-dependent TVCCM are synthesized along a nominal trajectory. Thus, nominal planning requires a certified tube bound, while controller synthesis requires the nominal plan around which the tube is to be constructed.

For TVLQR, a natural practical strategy is to alternate between these two steps. One first selects an initial erosion profile, solves \eqref{eq: det_STLTO} to obtain a nominal trajectory, synthesizes the TVLQR tracker along that trajectory, computes the resulting bound $r_{\delta,t}$, and then resolves the nominal planning problem with the updated erosion. This process is repeated until the bound used in planning dominates the bound certified by the resulting controller, i.e., the final planning erosion must be larger than the controller-induced tube.

If TVCCM is synthesized only along a given nominal trajectory, the same alternating strategy applies. However, TVCCM also admits a less coupled alternative. Inspired by~\cite{singh2017robust}, one may prescribe a constant target contraction rate $c<0$ and synthesize a time-varying dual metric and differential controller together with the uniform metric bounds \eqref{eq:continuous_tvccm_metric_bound}.
Assume that the continuous-time TVCCM condition \eqref{eq:continuous_tvccm_lmi} is feasible with this prescribed rate $c$. Since $M_t=W_t^{-1}$, the metric bounds imply
\begin{equation*}
\frac{1}{\bar\beta}I \preceq M_t \preceq I,
\qquad
\|M_t\|\le 1,
\qquad \forall t\in[0,T].
\end{equation*}
Therefore, $\overline{\sigma}_t$ in Theorem~\ref{thm: prob_bound} satisfies
\begin{equation*}
\overline{\sigma}_t=\sqrt{\|M_t\|}\sigma \le \sigma,
\end{equation*}
and the constant contraction rate specialization of Theorem~\ref{thm: prob_bound} yields the metric-space tube bound
\begin{equation}
\label{eq:global_tvccm_metric_bound}
r_{\delta,t}^{M}
\le
\frac{
\sigma\!\left(
\sqrt{1-e^{2ct}}
+
\sqrt{e^{-2c\Delta t}-1}
\right)
}{
\sqrt{-2c}
}
\sqrt{\varepsilon_1 n+\varepsilon_2\log\frac{2T}{\delta\Delta t}}.
\end{equation}
Consequently, the largest metric radius over the full horizon is
\begin{equation}
\label{eq:global_tvccm_radius}
\bar r_\delta^{M}
=
\frac{
\sigma\!\left(
\sqrt{1-e^{2cT}}
+
\sqrt{e^{-2c\Delta t}-1}
\right)
}{
\sqrt{-2c}
}
\sqrt{\varepsilon_1 n+\varepsilon_2\log\frac{2T}{\delta\Delta t}}.
\end{equation}
This bound is independent of the particular nominal trajectory and can therefore be computed before solving the nominal planning problem, provided the subsequent TVCCM synthesis is feasible with the prescribed $(c,\bar\beta)$. In addition, since $M_t \succeq \bar\beta^{-1}I$, any metric tube $\|e_t\|_{M_t}\le r$ implies
\begin{equation*}
\frac{1}{\bar\beta}\|e_t\|^2 \le \|e_t\|_{M_t}^2,
\end{equation*}
and hence the Euclidean tube bound
\begin{equation}
\label{eq:euclidean_bound}
\|e_t\| \le \sqrt{\bar\beta}\,\|e_t\|_{M_t}
\le
\sqrt{\bar\beta}\,r_{\delta,t}^{M}.
\end{equation}
This bound can be directly used to apply predicate erosion. The price paid for this decoupling is additional conservatism relative to trajectory-dependent TVCCM synthesis.

\begin{algorithm}[t]
\caption{Feedback Motion Planning Pipeline}
\label{alg:feedback_motion_planning}
\DontPrintSemicolon
\KwIn{System dynamics, STL formula $\varphi$, horizon $T$, risk level $\delta$, controller type}
\KwOut{Nominal trajectory $(\boldsymbol{x}^*,\boldsymbol{u}^*)$, tracking controller $K_t(\cdot;x_t^*)$}

\uIf{trajectory-independent TVCCM is selected}{
Choose design parameters $(c,\bar\beta)$ and solve the TVCCM synthesis problem subject to \eqref{eq:continuous_tvccm_lmi}\eqref{eq:continuous_tvccm_metric_bound}\;
Compute the nominal-independent erosion radius from \eqref{eq:euclidean_bound}\;
Solve the deterministic surrogate problem \eqref{eq: det_STLTO} using the induced eroded STL formula\;
Synthesize the TVCCM tracker along the resulting nominal trajectory and verify feasibility\;
}
\Else{
Initialize a planning erosion profile $\hat r_{\delta,t}^{(0)}$\;
\For{$i=0,1,2,\dots$ until convergence}{
Solve the deterministic surrogate problem \eqref{eq: det_STLTO} using $\hat r_{\delta,t}^{(i)}$\;

Solve the trajectory-dependent TVLQR/TVCCM synthesis problem along $(\boldsymbol{x}^{*,(i)},\boldsymbol{u}^{*,(i)})$\;
Compute the controller-induced PRT radius $r_{\delta,t}^{(i)}$ from Theorem~\ref{thm: prob_bound}\;
Update $\hat r_{\delta,t}^{(i+1)} \leftarrow r_{\delta,t}^{(i)}$\;
\If{$\hat r_{\delta,t}^{(i+1)} \leq r_{\delta,t}^{(i)}$}{
break\;
}
}
}
\Return{$(\boldsymbol{x}^*,\boldsymbol{u}^*)$, $K_t(\cdot;x_t^*)$}
\end{algorithm}

We summarize the overall algorithm in Algorithm~\ref{alg:feedback_motion_planning}. In practice, both the nominal trajectory optimization and the tracking controller synthesis problems are discretized before they can be solved numerically. We discuss the implementations in the next section.

\section{Numerical Optimization}
\label{sec:numerical_optimization}

We defined the stochastic STL trajectory optimization problem, introduced predicate erosion, derived the contraction-based PRT bound, and formulated feedback tracking controllers. The goal of this section is to convert these ingredients into \emph{computable numerical algorithms}. To do so, we introduce support times, discretize the nominal trajectory optimization problem, and then describe how the tracking controller constructions are implemented numerically on the support grid.

\subsection{Support-Time Discretization}

The deterministic surrogate problem \eqref{eq: det_STLTO} is infinite-dimensional. We therefore choose support times
\begin{equation*}
0=t_0<t_1<\cdots<t_N=T,
\qquad
h_k:=t_{k+1}-t_k,
\end{equation*}
and represent the continuous state and control trajectories by the discrete variables
\begin{equation*}
x_k := x(t_k), \quad u_k := u(t_k), \quad k=0,\ldots,N.
\end{equation*}
The nominal plan is thus encoded by the support-point sequences
\begin{equation*}
\boldsymbol{x}_{0:N} := (x_0,\ldots,x_N), \quad
\boldsymbol{u}_{0:N-1} := (u_0,\ldots,u_{N-1}).
\end{equation*}

Any consistent transcription of the dynamics can then be used. For instance, one may employ direct shooting, trapezoidal collocation, or Hermite-Simpson collocation~\cite{kelly2017introduction}. We keep the presentation agnostic to the specific scheme and denote the resulting dynamics map by
\begin{equation*}
x_{k+1}=F_k(x_k,u_k), \qquad k=0,\ldots,N-1,
\end{equation*}
where $F_k$ is induced by the chosen integration or collocation method.

The STL formulas are also evaluated over the support times. More precisely, the STL formulas are interpreted over the support states $\boldsymbol{x}_{0:N}$ or $\boldsymbol{X}_{0:N}$. Each atomic predicate is evaluated at the support states, and the temporal operators are enforced over the corresponding support-time indices lying in the specified time intervals. Since bounded-horizon STL depends only on finitely many time instants once a support grid is fixed, this yields a finite-dimensional constraint system suitable for numerical optimization. 

After predicate erosion, the remaining dynamics constraint is deterministic, so the continuous-time problem \eqref{eq: det_STLTO} can be handled by standard trajectory-optimization transcription techniques. For standard transcriptions such as Runge--Kutta and collocation methods, the resulting finite-dimensional nonlinear program approximates the original continuous-time optimal control problem, and under suitable regularity assumptions, local solutions of the discretized problem converge to local solutions of the continuous-time problem as the mesh is refined~\cite{dontchev2000second,kameswaran2008convergence}. In practice, discretization error can be further controlled by mesh-refinement procedures based on residual estimates~\cite{patterson2015ph,kelly2017introduction}.

\subsection{Trajectory Optimization under STL Specifications}

There are several ways to turn the support-time STL constraints into a numerical optimization problem. Mixed-integer formulations introduce binary variables and are exact for suitable classes of systems and predicates~\cite{belta2019formal,raman2014model}, but they become difficult to scale to nonlinear systems and long horizons. Another option is to replace the max/min operations in the robustness semantics with smooth approximations. 
Along this line, \cite{gilpin2020smooth} replaces the nonsmooth $\min$ and $\max$ operators in STL robustness by differentiable surrogates, which is suitable for gradient-based optimization but generally introduce approximation error. In this work, we instead adopt the exact smooth reformulation proposed in~\cite{han2025exactsmoothreformulationstrajectory}. At a high level, this method augments the planning problem with auxiliary variables so that the $\min$ and $\max$ can be represented by additional smooth nonlinear constraints without changing the feasible set of the support-time STL problem. This choice lets us use standard nonlinear programming solvers while retaining exactness.

Using this reformulation, the support-time nominal trajectory optimization problem can be written as
\begin{subequations}\label{eq:exact_smooth_nlp}
\begin{align}
&\min_{\boldsymbol{x}_{0:N},\,\boldsymbol{u}_{0:N-1},\,z}\quad 
 \sum_{k=0}^{N-1} L_{t_k}(x_k,u_k) + \Phi_T(x_N)
\label{eq:exact_smooth_nlp_obj}\\
\text{s.t.}\quad
& x_0 = x_0^{\mathrm{given}}, \\
& x_{k+1} = F_k(x_k,u_k), \qquad k=0,\ldots,N-1, \\
& u_k \in \mathcal{U}, \qquad k=0,\ldots,N-1, \\
& (\boldsymbol{x}_{0:N},z)\ \text{satisfy smooth encoding of } \tilde{\varphi}.
\label{eq:exact_smooth_nlp_stl}
\end{align}
\end{subequations}
Here $z$ collects the auxiliary variables introduced by the smooth STL transcription. By construction, \eqref{eq:exact_smooth_nlp_stl} is equivalent to enforcing the eroded STL specification on the support trajectory. The resulting finite-dimensional problem is a smooth nonlinear program and can be solved by off-the-shelf NLP solvers, such as SNOPT~\cite{gill2005snopt}.

\subsection{Numerical Implementation of Feedback Controllers}

Once a nominal plan $(\boldsymbol{x}_{0:N}^*,\boldsymbol{u}_{0:N-1}^*)$ has been obtained, the continuous-time feedback laws of Section~\ref{sec: feedback control} must also be implemented numerically.

\subsubsection{TVLQR}

For TVLQR, we evaluate the Jacobians
\begin{equation*}
A_k := \frac{\partial f}{\partial x}\bigg|_{x_k^*,u_k^*,t_k},
\qquad
B_k := \frac{\partial f}{\partial u}\bigg|_{x_k^*,u_k^*,t_k},
\end{equation*}
along the nominal support trajectory. We then integrate the continuous-time Riccati equation \eqref{eq: cdre} backward in time numerically from $t=T$ to $t=0$ to obtain support-time matrices
\begin{equation*}
S_k \approx S(t_k), \qquad k=0,\ldots,N.
\end{equation*}
The feedback gains at the support times are computed as
\begin{equation*}
K_k = R_k^{-1}B_k^\top S_k.
\end{equation*}
During rollout, these gains can be applied through zero-order hold or interpolation between consecutive support times. The sampled Riccati matrices also provide the support-time metric values
\begin{equation*}
M_k := S_k,
\end{equation*}
from which numerical lower and upper metric bounds can be extracted for use in the PRT computation.

\subsubsection{TVCCM}

For TVCCM, we discretize the continuous-time synthesis problem \eqref{eq:continuous_tvccm_lmi} on the support-time grid. Let
\begin{equation*}
A_k := \frac{\partial f}{\partial x}\bigg|_{x_k^*,u_k^*,t_k},
\qquad
B_k := \frac{\partial f}{\partial u}\bigg|_{x_k^*,u_k^*,t_k}.
\end{equation*}
Approximating the metric derivative in \eqref{eq:continuous_tvccm_lmi} by a finite difference,
\begin{equation*}
\dot{W}_{t_k} \approx \frac{W_{k+1}-W_k}{h_k},
\end{equation*}
we obtain the sampled contraction constraints
\begin{equation}
\label{eq:sampled_tvccm_lmi}
-\frac{W_{k+1}-W_k}{h_k}
+ A_kW_k + W_kA_k^\top
+ B_kY_k + Y_k^\top B_k^\top
\preceq - 2cW_k,
\end{equation}
for $k=0,\ldots,N-1$. To control the metric size and conditioning, we also impose
\begin{equation*}
I \preceq W_k \preceq \bar{W},
\qquad
\bar{W}\preceq \bar{\beta}I,
\qquad k=0,\ldots,N.
\end{equation*}
This leads to a semidefinite program of the form
\begin{subequations}\label{eq:tvccm_sdp}
\begin{align}
\min_{\{W_k,Y_k\},\,\bar{W},\,\bar{\beta}}\quad
& \bar{\beta} + w\,\mathrm{tr}(P\bar{W}P^\top) \\
\text{s.t.}\quad
& \eqref{eq:sampled_tvccm_lmi} \text{ holds for } k=0,\ldots,N-1, \\
& I \preceq W_k \preceq \bar{W}, \qquad k=0,\ldots,N, \\
& \bar{W}\preceq \bar{\beta}I,
\end{align}
\end{subequations}
where $w>0$ is a weight to adjust the priority of the two cost terms. We want to control both of the two cost terms because the radius of the PRT projected on the task-relevant coordinates depends on both the lower bound and the upper bound of the metric. The semidefinite program can be solved by any off-the-shelf solver, such as MOSEK~\cite{mosek}.

After solving \eqref{eq:tvccm_sdp}, we recover the metrics and feedback gains via
\begin{equation*}
M_k = W_k^{-1},
\qquad
K_k = Y_kW_k^{-1}.
\end{equation*}
The gains are then applied between support times by zero-order hold or interpolation.

\section{Simulation Experiments}
\label{sec:experiments}

We evaluate the proposed framework on four simulated benchmarks: a double integrator, a Dubins car, a PVTOL aircraft, and a 3D quadrotor model. The system dynamics and STL tasks are introduced in Section~\ref{sec:simulation_benchmarks}. The nominal trajectory optimization problems are formulated using the mathematical programming interface in Drake~\cite{tedrake2019drake} and solved with SNOPT~\cite{gill2005snopt}. In all simulation experiments, the nominal plan is tracked using a TVCCM controller. The corresponding semidefinite program for TVCCM synthesis is implemented in CVXPY~\cite{diamond2016cvxpy} and solved with MOSEK~\cite{mosek}.

For all benchmarks, let $\mathbf{p}(x)$ denote the position projection of the state.
For any set $\mathcal{S}$ in the position space, we define the atomic predicate $\pi_{\mathcal{S}} := \big(\mathbf{p}(x)\in \mathcal{S}\big)$,
equivalently $\mu_{\mathcal{S}}(x)\ge 0$ for the corresponding set function.
For the cost function, we use a quadratic running cost on the control input. 
We consider a risk level $\delta=10^{-3}$. All experiments are conducted on a desktop computer with an Intel Core i7-9700 CPU and 64GB of RAM.

Figure~\ref{fig:sim_benchmarks} visualizes the nominal plans and stochastic rollouts for the four benchmarks. In each case, the stochastic trajectories remain close to the nominal plan while satisfying the original STL task. Section~\ref{sec:deterministic_baselines} compares empirical STL satisfaction probabilities across planning strategies. The proposed method achieves $100\%$ satisfaction on all four benchmarks over $10000$ stochastic rollouts per task, whereas baselines that do not explicitly handle the chance constraint exhibit substantially lower satisfaction rates, as shown in Table~\ref{tab:satisfaction_probability}. Section~\ref{sec:chance_constraint_baselines} further compares the conservativeness of different STL chance-constraint handling schemes, as shown in Figure~\ref{fig:chance_decomposition_comparison}.

\begin{figure*}[t]
    \centering
    \includegraphics[width=\linewidth]{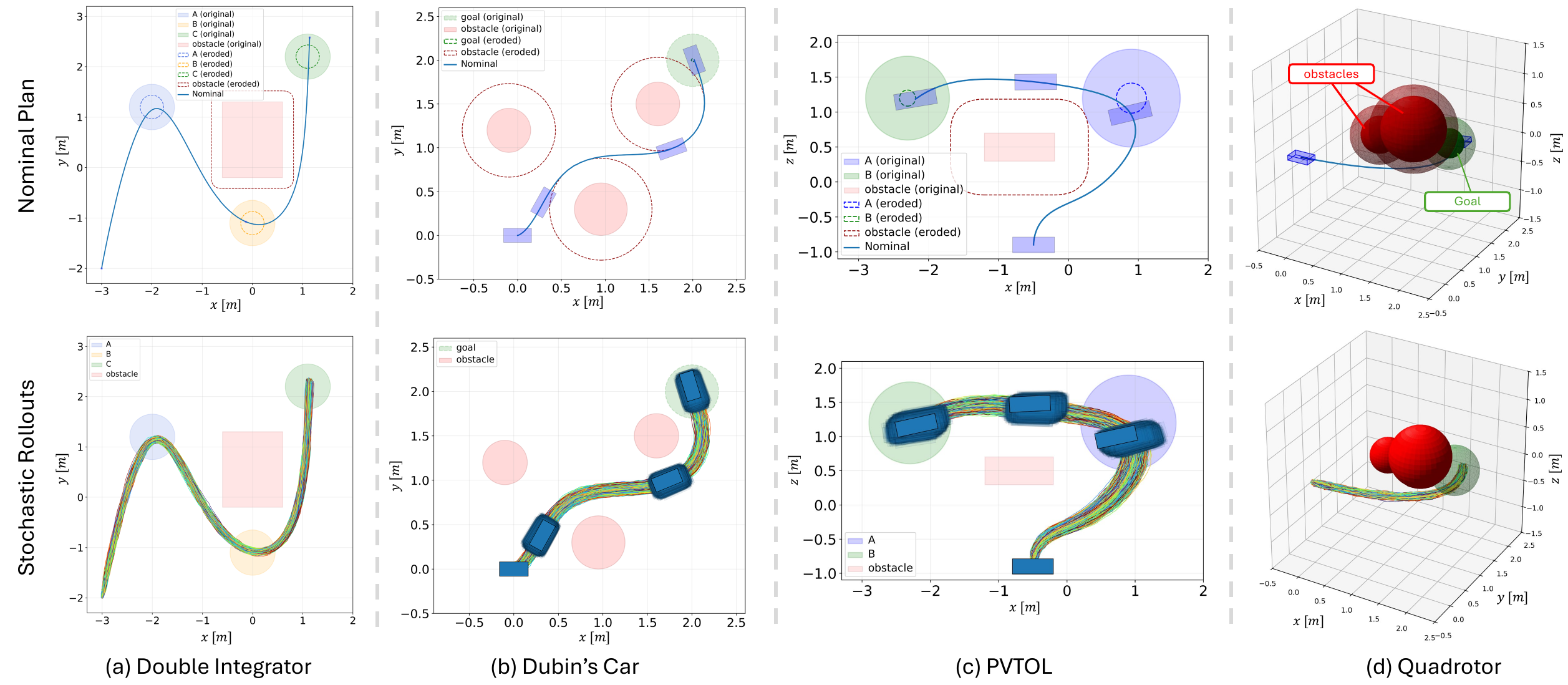}
    \caption{\textbf{Nominal plans and stochastic rollouts for the simulated benchmarks.} The top row shows the nominal trajectories together with the original and eroded predicates used in planning. The bottom row shows stochastic closed-loop rollouts under the synthesized feedback controller. Across all four systems, the stochastic trajectories satisfy the original STL tasks.}
    \label{fig:sim_benchmarks}
\end{figure*}

\subsection{Benchmarks}
\label{sec:simulation_benchmarks}
\paragraph{Double Integrator (DI)}
\label{sec:exp_double_integrator}

We first consider the 2D double-integrator system with state
$x_t = [p_{x,t}, p_{y,t}, v_{x,t}, v_{y,t}]^\top$
and control
$u_t = [a_{x,t}, a_{y,t}]^\top$.
The stochastic dynamics are
\begin{equation}
\label{eq:exp_di_dynamics}
\mathrm{d}X_t =
\begin{bmatrix}
v_{x,t} \\
v_{y,t} \\
a_{x,t} \\
a_{y,t}
\end{bmatrix}\mathrm{d}t
+ G\,\mathrm{d}W_t,
\end{equation}
where $G=0.02\cdot I_4$.
The admissible control set is
$u_t \in [-3,3]^2$.
The initial state is
$x_0 = [-3.0, -2.0, 0.0, 0.0]^\top$.

The STL task requires the trajectory to visit three goal regions in sequence while always avoiding an obstacle. Let
\begin{align*}
\mathcal{A} &= \{x:\|\mathbf{p}(x)-(-2.0,1.2)\|\le 0.45\}, \\
\mathcal{B} &= \{x:\|\mathbf{p}(x)-(0.0,-1.1)\|\le 0.45\}, \\
\mathcal{C} &= \{x:\|\mathbf{p}(x)-(1.1,2.2)\|\le 0.45\}.
\end{align*}
and let $\mathcal{O}$ denote the rectangular obstacle
$[-0.6,0.6]\times[-0.2,1.3]$.
The STL specification is
\begin{align*}
\label{eq:di_stl}
\varphi_{\mathrm{DI}}
=
(\neg \pi_{\mathcal{B}})\Until_{[0,T]}\pi_{\mathcal{A}}
\; \wedge\;
(\neg \pi_{\mathcal{C}})\Until_{[0,T]}\pi_{\mathcal{B}}
\; \wedge\;
\Eventually_{[0,T]}\pi_{\mathcal{C}} \\
\; \wedge\;
\Always_{[0,T]}\neg \pi_{\mathcal{O}}.
\end{align*}
In words, the double-integrator must first reach region $\mathcal{A}$, then reach region $\mathcal{B}$, and finally reach region $\mathcal{C}$, while remaining outside the obstacle throughout the entire horizon. The nominal planner uses horizon 8s with step size 0.1s. Figure~\ref{fig:sim_benchmarks}(a) shows that the planned trajectory threads the three goals while maintaining clearance from the obstacle. Over $10000$ stochastic rollouts, the closed-loop system satisfies the original STL specification in all runs.

\paragraph{Dubins Car}
\label{sec:exp_dubins}

We next consider the Dubins car with state
$x_t = [p_{x,t}, p_{y,t}, \theta_t, v_t]^\top$
and control
$u_t = [a_t, \omega_t]^\top$.
Its stochastic dynamics are
\begin{equation}
\label{eq:exp_dubins_dynamics}
\mathrm{d}X_t =
\begin{bmatrix}
v_t \cos \theta_t \\
v_t \sin \theta_t \\
\omega_t \\
a_t
\end{bmatrix}\mathrm{d}t
+ G\,\mathrm{d}W_t,
\end{equation}
where $G=\mathrm{diag}(0.02, 0.02, 0.01, 0.02)$.
The admissible control set is
$u_t \in [-2,2]\times[-0.8,0.8]$,
and the initial state is
$x_0 = [0.0, 0.0, 0.0, 0.0]^\top$.

We consider a reach-avoid task. Let
\[
\mathcal{G}=\{x:\|\mathbf{p}(x)-(2.0,2.0)\|\le 0.3\},
\]
and let the obstacle sets be
\begin{align*}
\mathcal{O}_1 &= \{x:\|\mathbf{p}(x)-(0.95,0.3)\|\le 0.3\}, \\
\mathcal{O}_2 &= \{x:\|\mathbf{p}(x)-(-0.1,1.2)\|\le 0.25\}, \\
\mathcal{O}_3 &= \{x:\|\mathbf{p}(x)-(1.6,1.5)\|\le 0.25\}.
\end{align*}
The STL specification is
\begin{equation}
\label{eq:dubins_stl}
\varphi_{\mathrm{Dubins}}
=
\Eventually_{[0,T]}\pi_{\mathcal{G}}
\;\wedge\;
\Always_{[0,T]}
\big(\neg \pi_{\mathcal{O}_1}\wedge \neg \pi_{\mathcal{O}_2}\wedge \neg \pi_{\mathcal{O}_3}\big).
\end{equation}
In words, the Dubins car must eventually reach the goal region while avoiding all three obstacle disks at all times. The nominal planner uses horizon 5s with step size 0.1s. Figure~\ref{fig:sim_benchmarks}(b) shows that the nominal plan navigates a narrow corridor between obstacles. Over $10000$ stochastic rollouts, the closed-loop system satisfies the original STL specification in all runs.

\paragraph{PVTOL}
\label{sec:exp_pvtol}

The PVTOL system is a more challenging underactuated nonlinear system. 
The state is
$x_t = [x_t, z_t, \phi_t, v_{x,t}, v_{z,t}, r_t]^\top$,
where $r_t = \dot{\phi}_t$ is the angular rate, and the control input is
$u_t = [u_{\ell,t}, u_{r,t}]^\top$.
The stochastic dynamics are
\begin{equation}
\label{eq:exp_pvtol_dynamics}
\mathrm{d}X_t =
\begin{bmatrix}
v_{x,t}\cos \phi_t - v_{z,t}\sin \phi_t \\
v_{x,t}\sin \phi_t + v_{z,t}\cos \phi_t \\
r_t \\
v_{z,t} r_t - g \sin \phi_t \\
-\,v_{x,t} r_t - g \cos \phi_t + \frac{u_{\ell,t}+u_{r,t}}{m} \\
\frac{l}{J}(u_{\ell,t}-u_{r,t})
\end{bmatrix}\mathrm{d}t
+ G\,\mathrm{d}W_t,
\end{equation}
with parameters
$m=0.486$, $J=0.00383$, $g=9.81$, and $l=0.25$.
We set $G=0.03 \cdot I_6$.
The initial state is
$x_0 = [-0.5, -0.9, 0.0, 0.0, 0.0, 0.0]^\top$.

The PVTOL system is tasked to pass region A before entering region B while avoiding obstacles in the $x$-$z$ plane.
Let
\begin{align*}
\mathcal{A} &= \{x:\|\mathbf{p}(x)-(0.9,1.2)\|\le 0.7\}, \\
\mathcal{B} &= \{x:\|\mathbf{p}(x)-(-2.3,1.2)\|\le 0.6\}.
\end{align*}
and let the obstacle be the rectangle
\[
\mathcal{O}=[-1.2,-0.2]\times[0.3,0.7].
\]
The STL specification is
\begin{equation}
\label{eq:pvtol_stl}
\varphi_{\mathrm{PVTOL}}
=
\big(\neg \pi_{\mathcal{B}}\big)\Until_{[0,T]}\pi_{\mathcal{A}}
\;\wedge\;
\Eventually_{[0,T]}\pi_{\mathcal{B}}
\;\wedge\;
\Always_{[0,T]}\neg \pi_{\mathcal{O}}.
\end{equation}
In words, the PVTOL must reach region $\mathcal{A}$ before entering region $\mathcal{B}$, while always avoiding the obstacle rectangle. The first conjunct enforces the ordering, and the second conjunct ensures that region $\mathcal{B}$ is eventually reached as part of the task. The nominal planner uses a horizon of 4s with a step size of 0.05s. Figure~\ref{fig:sim_benchmarks}(c) shows that the nominal trajectory passes region $\mathcal{A}$ while avoiding the obstacle before entering region $\mathcal{B}$. Over $10000$ stochastic rollouts, the closed-loop system satisfies the original STL specification in all runs.

\paragraph{3D Quadrotor (Quad)}
\label{sec:exp_quadrotor}

Finally, we evaluate the approach on a quadrotor flying in 3D space with state
$x_t = [p_{x,t}, p_{y,t}, p_{z,t}, v_{x,t}, v_{y,t}, v_{z,t}, \theta_{x,t}, \theta_{y,t}]^\top$
and control
$u_t = [a_{z,t}, \omega_{x,t}, \omega_{y,t}]^\top$.
Its stochastic dynamics are
\begin{equation}
\label{eq:exp_quadrotor_dynamics}
\mathrm{d}X_t =
\begin{bmatrix}
v_{x,t} \\
v_{y,t} \\
v_{z,t} \\
g \tan(\theta_{x,t}) \\
g \tan(\theta_{y,t}) \\
a_{z,t} \\
\omega_{x,t} \\
\omega_{y,t}
\end{bmatrix}\mathrm{d}t
+ G\,\mathrm{d}W_t,
\end{equation}
where $G=0.02 \cdot I_8.$

The initial condition is
$x_0 = 0$.
The admissible control set is given by the default bounds in the configuration.

The STL task is a 3D reach-avoid problem.
Let
\[
\mathcal{G}=\{x:\|\mathbf{p}(x)-(2.0,1.0,0.3)\|\le 0.4\},
\]
and let the obstacle balls be
\begin{align*}
\mathcal{O}_1 &= \{x:\|\mathbf{p}(x)-(0.5,1.5,-0.1)\|\le 0.3\}, \\
\mathcal{O}_2 &= \{x:\|\mathbf{p}(x)-(1.5,0.8,0.5)\|\le 0.5\}.
\end{align*}
The STL specification is
\begin{equation}
\label{eq:quadrotor_stl}
\varphi_{\mathrm{Quad}}
=
\Eventually_{[0,T]}\pi_{\mathcal{G}}
\;\wedge\;
\Always_{[0,T]}
\big(\neg \pi_{\mathcal{O}_1}\wedge \neg \pi_{\mathcal{O}_2}\big).
\end{equation}
In words, the quadrotor must eventually enter the goal ball in three-dimensional space while always avoiding both obstacle balls. The nominal planning horizon is 4s with step size 0.05s. Figure~\ref{fig:sim_benchmarks}(d) shows that the nominal trajectory bends around the 3D obstacles before entering the goal region. Over $10000$ stochastic rollouts, the closed-loop system satisfies the original STL specification in all runs.

\subsection{Comparison with Deterministic STL Planning Baselines}

\begin{table}[t]
\centering
\begin{tabular}{l|cccc}
\toprule
\textbf{Method} & \textbf{DI} & \textbf{Dubins} & \textbf{PVTOL} & \textbf{Quad} \\
\midrule
Non-Robust & $0.2\%$ & $2.6\%$ & $34.6\%$ & $50.0\%$ \\
SHMPC & $21.8\%$ & $22.4\%$ & $9.0\%$ & $77.6\%$ \\
Our Method & $100\%$ & $100\%$ & $100\%$ & $100\%$ \\
\bottomrule
\end{tabular}
\caption{\textbf{Specification satisfaction probability on the four simulation benchmarks.} The table reports the empirical probability that the original STL specification is satisfied under three planning strategies: non-robust STL planning, shrinking-horizon MPC (SHMPC), and the proposed method. The Non-Robust baseline uses the same feedback motion planning pipeline without predicate erosion, while SHMPC repeatedly replans from the observed stochastic state using deterministic STL constraints. The comparison shows that deterministic planning or replanning alone is insufficient for reliable STL satisfaction.}
\label{tab:satisfaction_probability}
\end{table}

\label{sec:deterministic_baselines}
Table~\ref{tab:satisfaction_probability} compares the proposed method with two baselines that do not explicitly handle the STL chance constraint. 
\paragraph{Non-Robust} The first baseline, denoted \emph{Non-Robust}, is our feedback motion planning pipeline without predicate erosion. It performs deterministic STL planning on the nominal system and then tracks the resulting plan with feedback in the stochastic system. This baseline shows that deterministic STL satisfaction alone is not enough: although the nominal trajectory satisfies the STL specification, stochastic fluctuations around that trajectory can still cause collisions with obstacles or failures to enter the goal region, leading to very low satisfaction probabilities in stochastic rollouts.

\paragraph{SHMPC} The second baseline adopts a shrinking-horizon model predictive control (SHMPC) strategy based on the shrinking-horizon scheme in~\cite{farahani2018shrinking}. At each feedback step, SHMPC solves a finite-horizon STL planning problem initialized at the newly observed stochastic state, applies only the first control input, and then replans with the remaining horizon shortened by one step. This repeated replanning gives SHMPC the ability to reject stochastic disturbances, and indeed its satisfaction probabilities are higher than those of the Non-Robust baseline (except for the PVTOL). However, because SHMPC in this comparison still uses deterministic STL constraints rather than considering an explicit chance constraint, it remains far from reliable. Moreover, the need to replan online at every step substantially increases computational burden during execution.

In contrast, our method explicitly accounts for stochastic deviations through predicate erosion and therefore achieves $100\%$ satisfaction on all four benchmarks. The comparison in Table~\ref{tab:satisfaction_probability} shows that deterministic feedback motion planning alone or replanning alone is insufficient. Explicit chance-constraint handling is necessary to obtain reliable STL satisfaction in the stochastic setting.

\subsection{Comparison on Chance-Constraint Handling}
\label{sec:chance_constraint_baselines}

\begin{figure}[t]
    \centering
    \includegraphics[width=\linewidth]{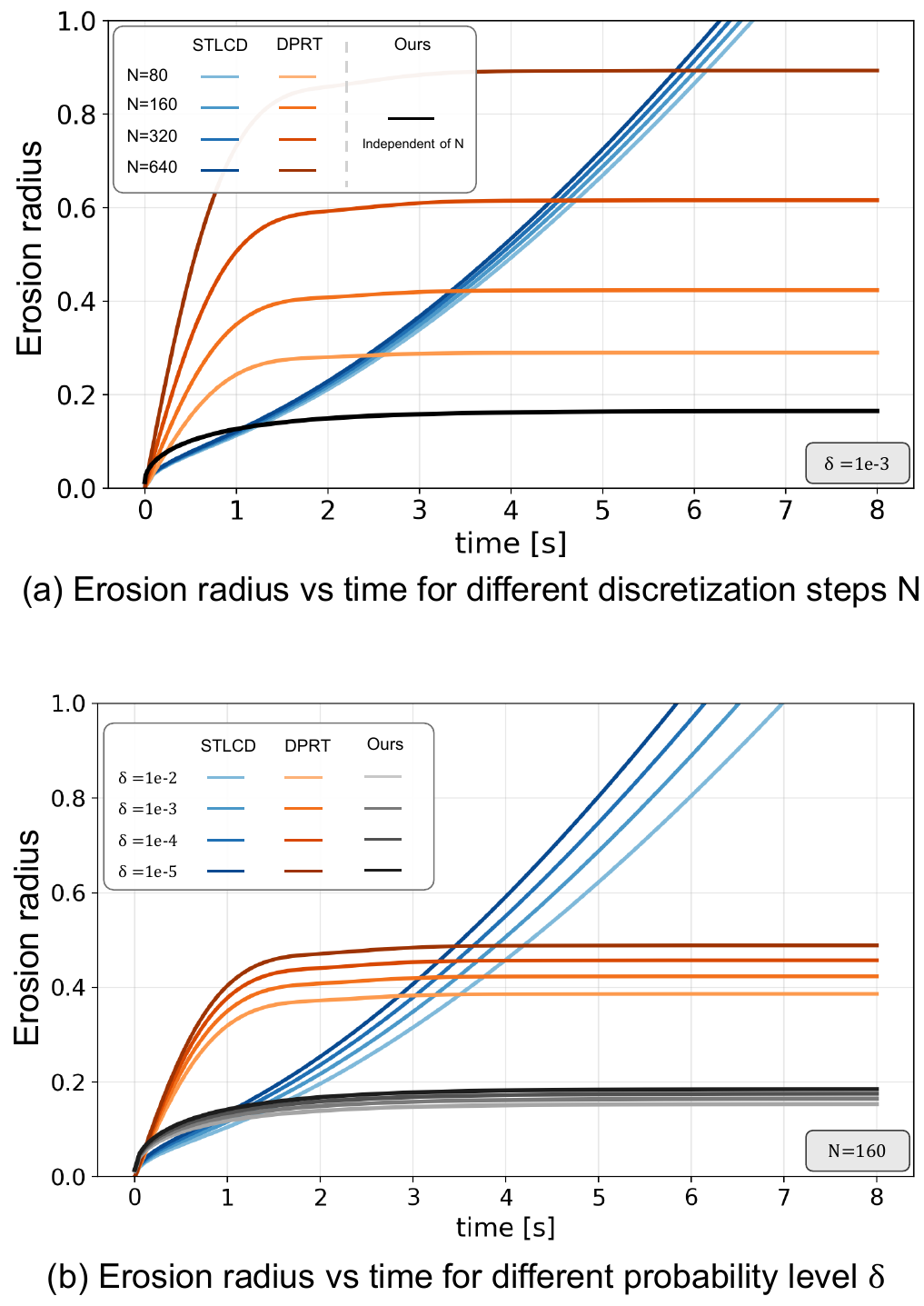}
\caption{\textbf{Comparison of chance-constraint tightening methods on a linear double-integrator benchmark.} (a) Erosion radius versus time under different numbers of discretization steps $N$. Our method provides a continuous-time bound and is therefore independent of $N$, whereas both discrete-time baselines become more conservative as $N$ increases. (b) Erosion radius versus time under different risk levels $\delta$. All methods require larger erosion for smaller $\delta$, but the dependence is significantly milder for our method. The comparison is restricted to a linear system because the baselines do not extend to the nonlinear systems.}
    \label{fig:chance_decomposition_comparison}
\end{figure}

Figure~\ref{fig:chance_decomposition_comparison} compares our chance-constraint handling scheme with two representative STL chance-constraint tightening baselines on a double-integrator reach-avoid task. Detailed system dynamics, task specification, and baseline derivations are deferred to Appendix~\ref{sec:appendix_baselines}. We restrict this comparison to the double-integrator benchmark because the baselines are only applicable to discrete-time linear systems, whereas our method naturally extends to continuous-time nonlinear systems.

The two baselines are discrete-time probabilistic reachable tube (DPRT) from \cite{vlahakis2024probabilistic} and STL-structure based chance decomposition (STLCD) from \cite{farahani2018shrinking}. Like our approach, DPRT constructs a PRT around a nominal trajectory and then tightens the predicates using the tube radius. The key difference is that DPRT works on a discrete-time support grid and therefore its tightening depends directly on the chosen temporal discretization. STLCD decomposes the STL chance constraint according to the logical structure of the formula and allocates probability budgets across predicates and time indices. Although conceptually different from tube-based methods, it ultimately also produces predicate-level tightening, so the three methods can be compared directly through their erosion radius.

Figure~\ref{fig:chance_decomposition_comparison}(a) plots the erosion radius as a function of time for different numbers of discretization steps $N$. Our method provides a continuous-time tube bound and therefore yields a single curve that is independent of $N$. In contrast, both discrete-time baselines become more conservative as $N$ increases, since finer discretization introduces more support times at which the chance constraint must be enforced. The two baselines also exhibit different temporal behavior: DPRT remains essentially flat after a short transient but at a much larger radius than our method if $N$ is large, whereas STLCD keeps growing over time and eventually becomes the most conservative of the three.

Figure~\ref{fig:chance_decomposition_comparison}(b) plots the erosion radius under different risk levels $\delta$. For all methods, the erosion radius increases as $\delta$ decreases, since a higher confidence level requires a larger safety margin. However, the growth is much milder for our method because the PRT bound depends on $\delta$ only through $\sqrt{\log(1/\delta)}$. The same qualitative separation remains visible: DPRT stays more conservative than our bound, while STLCD grows the fastest and quickly reaches radii that would make the tightened planning problem infeasible.

Overall, this comparison shows that, even on linear systems, our continuous-time predicate-erosion framework is less conservative than existing STL chance-constraint decomposition methods. Moreover, while these baselines are restricted to linear systems, our method extends naturally to the nonlinear systems considered throughout the rest of the paper.

\section{High-level Planning for Legged Robots}
\label{sec:experiments-quadruped}

\begin{figure*}[t]
    \centering
    \includegraphics[width=0.8\linewidth]{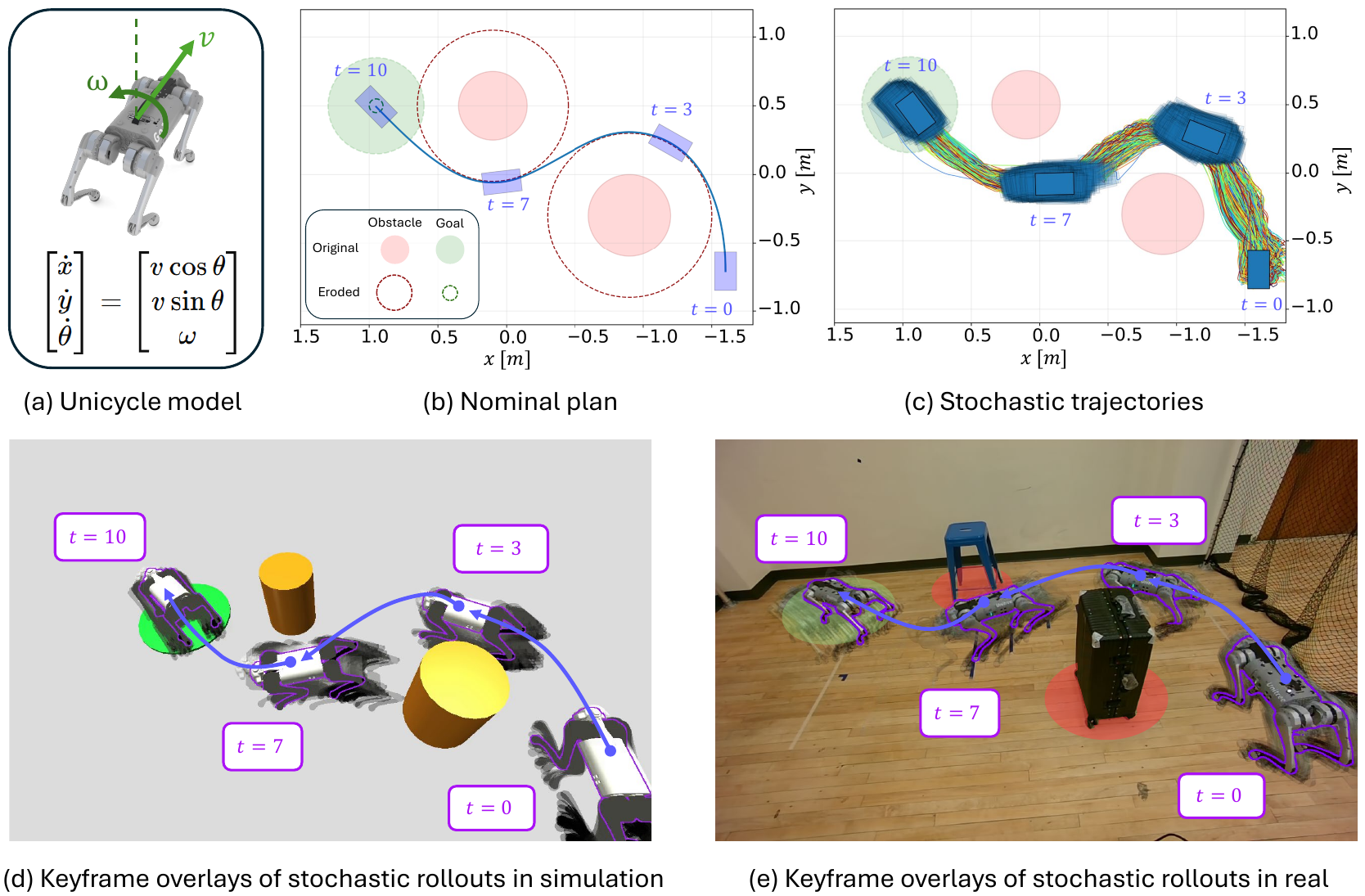}
    \caption{\textbf{Quadrupedal robot experiments on the Go1 reach-avoid task.}
    (a) High-level abstraction of the quadruped by a kinematic unicycle model with forward-speed and yaw-rate commands. 
    (b) Nominal plan generated on the unicycle abstraction under the eroded STL specification. 
    (c) Stochastic closed-loop rollouts in simulation under the feedback controller, showing $5000$ trajectories of the quadruped base projected onto the planar task space. 
    (d) Keyframe overlay of representative stochastic rollouts in IsaacGym. 
    (e) Keyframe overlay of 50 hardware executions on the real Unitree Go1. One representative trajectory is highlighted in purple and other trajectories are overlaid transparently.}

    \label{fig:go1_reachavoid}
\end{figure*}

\begin{figure}[t]
    \centering
    \includegraphics[width=0.7\linewidth]{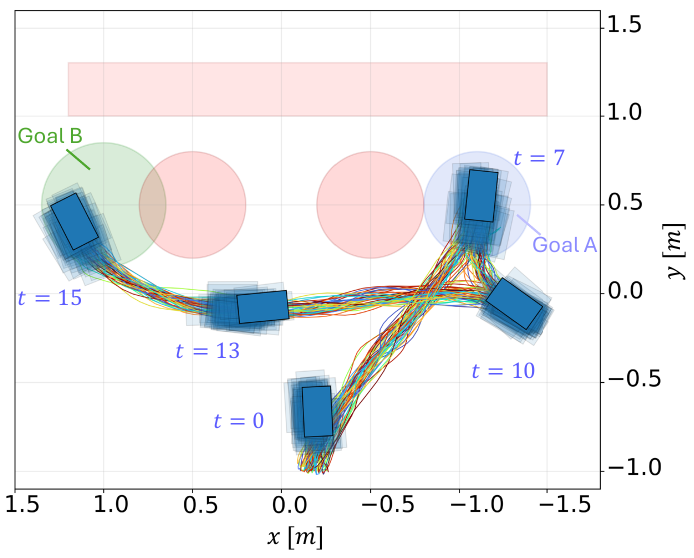}
    \caption{\textbf{Real Go1 pass-before trajectories.} The figure shows $50$ hardware executions of the pass-before task under the feedback controller. The corresponding nominal plan and keyframe overlay are shown in Fig.~\ref{fig:teaser}.}
    \label{fig:go1_passbefore_real}
\end{figure}

We demonstrate the applicability of the proposed framework on a quadrupedal robot controlled through a target base-velocity interface. Rather than planning directly in the full legged dynamics, we abstract the robot base with a unicycle model, which captures the high-level motion relevant to navigation while remaining tractable for STL motion planning. The resulting nominal unicycle trajectory is then tracked by a feedback controller that outputs target linear and angular velocities, which are passed to the low-level velocity-tracking policy of the quadruped. An illustration of the overall pipeline is shown in Fig.~\ref{fig:go1_reachavoid}.

This experiment highlights the main advantage of feedback motion planning: the nominal planner operates on a simplified model, while the feedback layer compensates for modeling mismatch between the abstraction and the actual robot.

The quadruped is modeled at the planning level by a stochastic kinematic unicycle with state
$x_t = [x_t, y_t, \theta_t]^\top$ and control $u_t = [v_t, \omega_t]^\top$, where
$v_t$ is the forward speed and $\omega_t$ is the yaw-rate command. The
dynamics are
\begin{equation}
\mathrm{d}X_t =
\begin{bmatrix}
v_t \cos \theta_t \\
v_t \sin \theta_t \\
\omega_t
\end{bmatrix}\mathrm{d}t
+ G\,\mathrm{d}W_t,
\end{equation}
where $W_t$ is a Wiener process and $G$ captures model mismatch. This mismatch can be treated as stochastic noise because the true legged dynamics are hybrid and include repeated impact events when the feet contact the ground. These impacts perturb the robot base and are well modeled as stochastic disturbances at the abstraction level.

The high-level feedback controller outputs $u=[v,\omega]^\top$, 
which is mapped to the Go1 target-velocity interface as
\[
[v_x,v_y,\text{yaw\_rate}] = [v,0,\omega].
\]

We evaluate the framework on two navigation tasks defined on this unicycle abstraction and used in both simulation and hardware experiments.

\paragraph{Go1 reach-avoid.}
Let
\[
\mathcal{G}=\{x:\|\mathbf{p}(x)-(0.5,0.95)\|\le 0.35\},
\]
and let the obstacle sets be
\begin{align*}
\mathcal{O}_1 &= \{x:\|\mathbf{p}(x)-(-0.3,-0.9)\|\le 0.35\}, \\
\mathcal{O}_2 &= \{x:\|\mathbf{p}(x)-(0.5,0.1)\|\le 0.30\}.
\end{align*}
The STL specification is
\begin{equation}
\label{eq:go1_reach_avoid_stl}
\varphi
=
\Eventually_{[0,T]}\pi_{\mathcal{G}}
\;\wedge\;
\Always_{[0,T]}\big(\neg \pi_{\mathcal{O}_1}\wedge \neg \pi_{\mathcal{O}_2}\big).
\end{equation}
In words, the robot must eventually reach the goal region while always avoiding both circular obstacles. The start state is $x_0=[-0.71,-1.6,0.0]^\top$. Figure~\ref{fig:go1_reachavoid} visualizes the unicycle abstraction, the nominal plan, stochastic simulation rollouts, and keyframe overlays for this task.

\paragraph{Go1 pass-before.}
Let
\begin{align*}
\mathcal{A} &= \{x:\|\mathbf{p}(x)-(0.5,-1.1)\|\le 0.30\}, \\
\mathcal{B} &= \{x:\|\mathbf{p}(x)-(0.5,1.0)\|\le 0.35\},
\end{align*}
let the box obstacle be
\[
\mathcal{O}_{\mathrm{box}}=[1.0,1.3]\times[-1.5,1.2],
\]
and let the circular obstacles be
\begin{align*}
\mathcal{O}_1 &= \{x:\|\mathbf{p}(x)-(0.5,-0.5)\|\le 0.30\}, \\
\mathcal{O}_2 &= \{x:\|\mathbf{p}(x)-(0.5,0.5)\|\le 0.30\}.
\end{align*}
The STL specification is
\begin{align*}
\label{eq:go1_pass_before_stl}
\varphi
=&
\big(\neg \pi_{\mathcal{B}}\big)\Until_{[0,T]}\pi_{\mathcal{A}}
\;\wedge\;
\Eventually_{[0,T]}\pi_{\mathcal{B}} \\
&\;\wedge\;
\Always_{[0,T]}\big(\neg \pi_{\mathcal{O}_{\mathrm{box}}}\wedge \neg \pi_{\mathcal{O}_1}\wedge \neg \pi_{\mathcal{O}_2}\big).
\end{align*}
In words, the robot must reach region $\mathcal{A}$ before entering region $\mathcal{B}$, while always avoiding the box obstacle and the two circular obstacles. The start state is $x_0=[-0.7,-0.2,0.0]^\top$. The corresponding nominal plan and hardware keyframe overlay are visualized in Fig.~\ref{fig:teaser}. The stochastic hardware trajectories for this task are visualized in Fig.~\ref{fig:go1_passbefore_real}.

\paragraph{Simulation setting}
We first evaluate these tasks in IsaacGym~\cite{makoviychuk2021isaacgymhighperformance}, which enables large-scale parallel rollout of the quadruped model. The high-level feedback controller produces target forward-speed and yaw-rate commands, which are executed by the low-level \emph{walk-these-ways} velocity-tracking policy~\cite{margolis2023walk}. To calibrate the abstraction noise, we collect rollout data under a diverse set of target-velocity commands and compare the observed next-step base motion with the one-step prediction of the unicycle model. We then estimate the covariance of these one-step residuals and set $G$ accordingly. For both tasks we use
\[
G=
\begin{bmatrix}
0.0091 & 0.0001 & 0.0003 \\
0.0001 & 0.0093 & 0.0006 \\
0.0003 & 0.0006 & 0.0214
\end{bmatrix}.
\]
Although this calibration process is empirical, we can also obtain statistical guarantees that this bound overapproximates the true noise through calibration techniques based on extreme value theory \cite{knuth2023statistical} or conformal prediction \cite{srinivasan2026safety}. For each task, we generate a nominal plan on the unicycle abstraction and then execute $5000$ stochastic rollouts in simulation. Both the reach-avoid and pass-before tasks achieve $100\%$ STL satisfaction in these simulation rollouts. Fig.~\ref{fig:go1_reachavoid} shows representative results for the reach-avoid task, including the nominal plan, stochastic trajectories, and simulation keyframe overlays.

\paragraph{Real-world setting}
We then deploy our method on a real-world Unitree Go1 robot to accomplish the same two STL tasks. The robot base state is estimated online using a Vicon motion capture system. In particular, we extract the planar position and yaw angle and use them as the unicycle state $x=[x,y,\theta]^\top$. For each task, we execute $50$ hardware runs under the feedback controller, all of which successfully complete the task. 
Fig.~\ref{fig:go1_reachavoid} shows the hardware keyframe overlay for the reach-avoid task, and Fig.~\ref{fig:teaser} shows the hardware keyframe overlay for the pass-before task. Fig.~\ref{fig:go1_passbefore_real} overlays $50$ hardware executions of the pass-before task.
The real-world trajectories are subjected to additional sources of uncertainty such as state estimation noise and actuation noise. Nevertheless, all hardware runs satisfy the STL specifications, demonstrating successful hardware deployment of the proposed feedback motion planning pipeline.

\section{Conclusion}

We present a feedback motion planning framework for continuous-time stochastic nonlinear systems under STL specifications. The central idea is to convert the original chance-constrained STL trajectory optimization problem into a deterministic surrogate by eroding the STL predicates using a PRT around a nominal trajectory. To make this reduction practical, we developed contraction-based PRT bounds together with feedback tracking controllers, and showed how these ingredients can be implemented with numerical optimization.
The resulting framework provides a systematic way to combine high-level temporal-logic planning with feedback control and probabilistic guarantees. The experiments demonstrate that the method scales across multiple simulation benchmark systems, is less conservative and achieves higher specification satisfaction probability than baselines, and transfers to a real quadrupedal platform. Future work will investigate tighter integration between planning, feedback synthesis, and uncertainty quantification.

\appendix
\subsection{Details on Chance-Constraint Handling in Baselines}
\label{sec:appendix_baselines}

We consider two representative stochastic STL baselines on a canonical double-integrator reach-avoid task. Although our method is formulated for continuous-time stochastic systems and provides a continuous-time trajectory-level guarantee, both baselines considered here are discrete-time methods. Therefore, for this comparison only, we discretize the double-integrator SDE and enforce the STL formula on the resulting support grid. This yields a simpler problem than the continuous-time chance-constrained STL problem addressed by our method. The purpose of this subsection is to make explicit how each baseline converts the STL chance constraint into deterministic optimization constraints.

Consider the discrete-time double-integrator model obtained from~\eqref{eq:exp_di_dynamics} under zero-order-hold control with step size $\Delta t$:
\begin{equation*}
\label{eq:baseline_di_discrete}
X_{k+1}=AX_k+Bu_k+w_k,\qquad w_k\sim\mathcal{N}(0,Q_d),
\end{equation*}
where
\begin{equation*}
A=
\begin{bmatrix}
1&0&\Delta t&0\\
0&1&0&\Delta t\\
0&0&1&0\\
0&0&0&1
\end{bmatrix},\qquad
B=
\begin{bmatrix}
\frac{1}{2}\Delta t^2&0\\
0&\frac{1}{2}\Delta t^2\\
\Delta t&0\\
0&\Delta t
\end{bmatrix}.
\end{equation*}
Since $G=0.02I_4$ in~\eqref{eq:exp_di_dynamics}, the exact discrete-time covariance is
\begin{equation*}
\label{eq:baseline_qd}
Q_d=0.02^2
\begin{bmatrix}
\Delta t+\frac{\Delta t^3}{3}&0&\frac{\Delta t^2}{2}&0\\
0&\Delta t+\frac{\Delta t^3}{3}&0&\frac{\Delta t^2}{2}\\
\frac{\Delta t^2}{2}&0&\Delta t&0\\
0&\frac{\Delta t^2}{2}&0&\Delta t
\end{bmatrix}.
\end{equation*}
Let
\begin{equation*}
\label{eq:baseline_projection_matrix}
P=\begin{bmatrix}I_2 & 0_{2\times 2}\end{bmatrix}
\end{equation*}
be the position projection, so that $p_k=P x_k$. Let $p_{\mathrm{goal}}$ and $R_{\mathrm{goal}}$ denote the center and radius of the goal disk, and let $p_{\mathrm{obs}}$ and $R_{\mathrm{obs}}$ denote the center and radius of the circular obstacle. We define the reach-avoid predicates
\begin{align*}
\pi_{\text{goal}}(k)&:=\left(\|P X_k-p_{\mathrm{goal}}\|_2\le R_{\mathrm{goal}}\right),\\
\pi_{\text{obs}}(k)&:=\left(\|P X_k-p_{\mathrm{obs}}\|_2\ge R_{\mathrm{obs}}\right),
\end{align*}
where $\pi_{\text{goal}}$ encodes reaching the goal region and $\pi_{\text{obs}}$ encodes avoiding the obstacle. The discrete-time reach-avoid STL formula is
\begin{equation}
\label{eq:baseline_ra_formula}
\varphi_{\mathrm{RA}}
:=
\Eventually_{[0,N]}\pi_{\text{goal}}
\;\wedge\;
\Always_{[0,N]}\pi_{\text{obs}}.
\end{equation}
The chance constraint is therefore
\begin{equation}
\label{eq:baseline_ra_chance}
\mathbb{P}\left((X_{0:N},0)\models \varphi_{\mathrm{RA}}\right)
\ge 1-\delta,
\end{equation}
where $\delta\in(0,1)$ is the desired risk level.

\subsubsection{STL-structure based chance decomposition (STLCD)}
\label{sec:baseline_farahani}
The first baseline, which we denote STL-structure based chance decomposition (STLCD), follows the chance-constraint decomposition used in~\cite{farahani2018shrinking}. The original method is developed for affine predicates so that each atomic predicate chance constraint can be converted into a linear deterministic constraint. Since our experiments use circular predicates and solve the resulting problems with nonlinear programming, we keep the same STL-level risk decomposition but convert the circular predicate chance constraints using Gaussian confidence balls.

For a fixed open-loop input sequence, the predicted state distribution satisfies
\begin{align}
\bar x_{0} &= x_0, & \Sigma_0&=0,\nonumber\\
\bar x_{k+1} &= A\bar x_k+Bu_k, & \Sigma_{k+1}&=A\Sigma_kA^\top+Q_d.
\label{eq:farahani_mean_cov}
\end{align}
Thus $Y_k:=P X_k$ is Gaussian with mean $\bar p_k=P\bar x_k$ and covariance $S_k=P\Sigma_kP^\top$. For any risk level $\eta\in(0,1)$, define the scalar position uncertainty radius
\begin{equation*}
\label{eq:farahani_radius}
r_k(\eta):=\sqrt{\chi^2_2(1-\eta)\,\lambda_{\max}(S_k)},
\end{equation*}
where $\chi^2_2(1-\eta)$ is the $(1-\eta)$-quantile of the chi-square distribution with two degrees of freedom. Then
\begin{equation*}
\mathbb{P}\left(\|Y_k-\bar p_k\|_2\le r_k(\eta)\right)\ge 1-\eta.
\end{equation*}
Consequently, the following deterministic constraints are sufficient for the two predicate-level chance constraints used in this reach-avoid task:
\begin{align}
\mathbb{P}\big(\pi_{\text{goal}}(k)\big)\ge 1-\eta
&\Leftarrow
\|\bar p_k-p_{\mathrm{goal}}\|_2\le R_{\mathrm{goal}}-r_k(\eta),
\label{eq:farahani_goal_det}\\
\mathbb{P}\big(\pi_{\text{obs}}(k)\big)\ge 1-\eta
&\Leftarrow
\|\bar p_k-p_{\mathrm{obs}}\|_2\ge R_{\mathrm{obs}}+r_k(\eta).
\label{eq:farahani_obs_det}
\end{align}

We next decompose the STL chance constraint. Let
\begin{equation*}
\varphi_{\text{goal}}:=\Eventually_{[0,N]}\pi_{\text{goal}},
\qquad
\varphi_{\text{obs}}:=\Always_{[0,N]}\pi_{\text{obs}}.
\end{equation*}
Choose outer risk allocations $\delta_{\text{goal}},\delta_{\text{obs}}\ge0$ satisfying
\begin{equation*}
\label{eq:outer_allocation}
\delta_{\text{goal}}+\delta_{\text{obs}}\le \delta.
\end{equation*}
It is sufficient to impose $\mathbb{P}(\varphi_{\text{goal}})\ge 1-\delta_{\text{goal}}$ and $\mathbb{P}(\varphi_{\text{obs}})\ge 1-\delta_{\text{obs}}$. In the baseline implementation, we use the uniform outer allocation
$\delta_{\text{goal}}=\delta_{\text{obs}}=\frac{\delta}{2}.$

For the always component, the union bound gives the predicate-level constraints
\begin{equation*}
\label{eq:farahani_always_allocation}
\mathbb{P}\big(\pi_{\text{obs}}(k)\big)\ge 1-\eta_{\text{obs},k},\qquad
\sum_{k=0}^{N}\eta_{\text{obs},k}\le \delta_{\text{obs}}.
\end{equation*}
With uniform allocation over time, $\eta_{\text{obs},k}=\delta_{\text{obs}}/(N+1)$, and~\eqref{eq:farahani_obs_det} gives
\begin{equation}
\label{eq:farahani_always_det}
\|\bar p_k-p_{\mathrm{obs}}\|_2\ge R_{\mathrm{obs}}+r_k\!\left(\frac{\delta_{\text{obs}}}{N+1}\right),
\qquad k=0,\ldots,N.
\end{equation}

For the eventually component, the goal only needs to be reached at one time instant. The resulting deterministic goal constraint is
\begin{equation}
\label{eq:farahani_eventually_det}
\exists j\in\{0,\ldots,N\}:\quad
\|\bar p_j-p_{\mathrm{goal}}\|_2
\le R_{\mathrm{goal}}-r_j(\delta_{\text{goal}}).
\end{equation}

Under the uniform outer allocation, the predicate-level risks become
$\eta_{\text{obs},k}=\frac{\delta}{2(N+1)}$.
Hence, the uniformly allocated deterministic constraints are
\begin{align}
&\|\bar p_k-p_{\mathrm{obs}}\|_2
\ge R_{\mathrm{obs}}+r_k\!\left(\frac{\delta}{2(N+1)}\right),
\; k=0,\ldots,N,
\label{eq:uniform_farahani_obs_det}\\
&\exists j\in\{0,\ldots,N\}:\quad
\|\bar p_j-p_{\mathrm{goal}}\|_2
\le R_{\mathrm{goal}}-r_j\!\left(\frac{\delta}{2}\right).
\label{eq:uniform_farahani_goal_det}
\end{align}
This baseline allocates risk at the STL-component level and then across the support times of the always operator.

\subsubsection{Discrete-time probabilistic reachable tube (DPRT)}
\label{sec:baseline_vlahakis}
The second baseline, which we denote discrete-time probabilistic reachable tube (DPRT), follows the probabilistic tube-based tightening approach of~\cite{vlahakis2024probabilistic}. Assume a stabilizing feedback gain $K$ is given. We decompose the stochastic state into a nominal component and an error component,
\begin{equation*}
X_k=z_k+e_k,
\end{equation*}
and apply the controller
\begin{equation*}
\label{eq:tube_controller}
u_k=v_k+K(X_k-z_k)=v_k+Ke_k.
\end{equation*}
Then
\begin{align}
z_{k+1}&=Az_k+Bv_k,
\label{eq:tube_nominal_dyn}\\
e_{k+1}&=A_\text{cl}e_k+w_k,
\qquad A_\text{cl}:=A+BK,
\label{eq:tube_error_dyn}
\end{align}
with $z_0=x_0$ and $e_0=0$.

The baseline constructs a probabilistic reachable tube for~\eqref{eq:tube_error_dyn}. For Gaussian disturbances, choose a disturbance confidence region
\begin{equation*}
\label{eq:disturbance_cr}
\mathcal{W}_{\beta}:=\left\{w\in\mathbb{R}^4:
 w^\top Q_d^{-1}w\le \chi^2_4(\beta)\right\}.
\end{equation*}
Choosing $\beta=1-\delta/N$ gives a union-bound trajectory confidence level of at least $1-\delta$. Starting from $\mathcal{E}_0=\{0\}$, propagate the error reachable sets by
\begin{equation}
\label{eq:error_set_recursion}
\mathcal{E}_{k+1}=A_\text{cl}\mathcal{E}_k\oplus \mathcal{W}_{\beta},
\qquad k=0,\ldots,N-1.
\end{equation}
Then the tube $\mathcal{E}:=\mathcal{E}_0\times\cdots\times\mathcal{E}_N$ satisfies $\mathbb{P}(e_{0:N}\in\mathcal{E})\ge 1-\delta.$
Define the projected tube radius
\begin{equation*}
\label{eq:projected_tube_radius}
r_k^{\mathrm{tube}}:=\max_{e\in\mathcal{E}_k}\|Pe\|_2.
\end{equation*}
If the nominal trajectory satisfies the predicates tightened by $r_k^{\mathrm{tube}}$, then every trajectory whose error remains in the tube satisfies the original predicates. For the circular goal and obstacle, this gives
\begin{align*}
\|Pz_k-p_{\mathrm{goal}}\|_2&\le R_{\mathrm{goal}}-r_k^{\mathrm{tube}}
&&\Rightarrow \pi_{\text{goal}}(k),\\
\|Pz_k-p_{\mathrm{obs}}\|_2&\ge R_{\mathrm{obs}}+r_k^{\mathrm{tube}}
&&\Rightarrow \pi_{\text{obs}}(k).
\end{align*}
Therefore the chance constraint~\eqref{eq:baseline_ra_chance} is conservatively enforced by the deterministic STL constraints
\begin{align}
&\exists k\in\{0,\ldots,N\}:
\|Pz_k-p_{\mathrm{goal}}\|_2\le R_{\mathrm{goal}}-r_k^{\mathrm{tube}},
\label{eq:tube_eventually_det}\\
&\|Pz_k-p_{\mathrm{obs}}\|_2\ge R_{\mathrm{obs}}+r_k^{\mathrm{tube}},
\qquad k=0,\ldots,N.
\label{eq:tube_always_det}
\end{align}
In contrast to the predicate-level baseline, this method does not allocate risk through the STL syntax. The entire probability budget is used to construct a trajectory-level tube, and the STL formula is then enforced deterministically on the tightened predicates.



\bibliographystyle{ieeetr}
\bibliography{references}
\end{document}